\def\extended{true}

\documentclass[conference]{IEEEtran}
\IEEEoverridecommandlockouts
\usepackage{cite}
\usepackage{amsmath,amssymb,amsfonts}
\usepackage{graphicx}
\usepackage{textcomp}
\usepackage[english]{babel}
\usepackage{hyperref}

\newcommand{\ifextended}[2]{\ifdefined\extended{#1}\else{#2}\fi}

\ifdefined\extended
\newcommand{\extref}[1]{\ref{#1}}
\newcommand{\extpageref}[1]{\pageref{#1}}
\else
\usepackage{xr}
\newcommand{\extref}[1]{\ref{ext-#1}}
\newcommand{\extpageref}[1]{\pageref{ext-#1}}
\fi

\usepackage{siunitx}
\usepackage{booktabs}
\usepackage{scalefnt}

\usepackage{amsthm,mathtools}
\allowdisplaybreaks

\newtheorem{proposition}{Proposition}
\newtheorem{definition}{Definition}

\newcommand{\eqdef}     {\stackrel{{\textrm{\rm\tiny def}}}{=}}

\DeclareMathOperator*{\argmax}{arg\,max}

\usepackage[normalem]{ulem}
\newcommand{\hide}[1]{}

\graphicspath{{./}}

\usepackage{flushend}

\usepackage[dvipsnames]{xcolor}
\definecolor{pink}{rgb}{0.858, 0.188, 0.478}
\newcommand{\persComment}[3]{
  \ifmmode
  \text{\textcolor{#3}{[#2] \em #1}}
  \else
  \textcolor{#3}{[#2] \em #1}
  \fi
}

\def\ie{{\em i.e.}\xspace}
\def\eg{{\em e.g.}\xspace}
\def\cf{{\em cf.}\xspace}

\def\reals{{\mathbb R}}
\def\cS{{\cal S}}
\def\cA{{\cal A}}

\def\cI{{\cal I}}
\def\cZ{{\cal Z}}
\def\cE{{\cal E}}

\def\nNI{\#ni} \def\fsc{\mathit{fsc}} \def\FSC{\mathcal{FSC}} \newcommand{\bb}[2]{b_{#1}^{#2}} 

\usepackage[ruled,vlined,linesnumbered]{algorithm2e}
\SetKwProg{Fct}{Fct}{}{}

\usepackage[numbers]{natbib}
\newcommand{\infJESP}[1][]{Inf-JESP{#1}\xspace}

\begin{document}

\title{Solving infinite-horizon Dec-POMDPs\\
  using Finite State Controllers within JESP
\thanks{This work was supported by the French National Research
    Agency (ANR) through the “Flying Coworker” Project under Grant 18-CE33-0001.}
}

\author{\IEEEauthorblockN{Yang You, Vincent Thomas, Francis Colas and Olivier Buffet}
\IEEEauthorblockA{Université de Lorraine, INRIA, CNRS, LORIA\\
      F-54000 Nancy, France \\Email: firstname.lastname@loria.fr}
  }

\maketitle

\begin{abstract}
{\em
This paper looks at solving collaborative planning problems formalized as Decentralized POMDPs (Dec-POMDPs) by searching for Nash equilibria, \ie{}, situations where each agent's policy is a best response to the other agents' (fixed) policies.
While the Joint Equilibrium-based Search for Policies (JESP) algorithm does this in the finite-horizon setting relying on policy trees, we propose here to adapt it to infinite-horizon Dec-POMDPs by using finite state controller (FSC) policy representations.
In this article, we
(1) explain how to turn a Dec-POMDP with $N-1$ fixed FSCs into an infinite-horizon POMDP whose solution is an $N^\text{th}$ agent best response;
(2) propose a JESP variant, called \infJESP, using this to solve infinite-horizon Dec-POMDPs;
(3) introduce heuristic initializations for JESP aiming at leading to good solutions; and
(4) conduct experiments on state-of-the-art benchmark problems to evaluate our approach.
}
\end{abstract}


\ifextended{\noindent [Note: This extended version of the ICTAI 2021 submission
  contains supplemental material in appendices.]}{}

\section{Introduction}

Decentralized Partially Observable Markov Decision Problems (Dec-POMDPs)  represent multi-agent sequential decision problems where the objective is to derive the policies of multiple agents so that their decentralized execution maximizes the average cumulative reward.
Each agent policy can rely only on  his individual history, namely the sequence of his past actions and observations.

Solving a finite-horizon Dec-POMDP has been proven to be NEXP in the worst case \cite{Bernstein02}, even for two agents, constraining the possible efficiency of optimal solvers for generic Dec-POMDPs.
The main difficulties of solving a Dec-POMDP lie in two facts: (i) the state of the system evolves according to the actions of all the agents; and (ii) the action performed by each agent should be based only on his own history.
Thus, all policies are interdependent: each agent's optimal decision depends on
the other agents' possible current histories and future policies.

To circumvent these interdependencies during the optimization process,  the JESP algorithm (Joint Equilibrium-Based Search for Policies) \cite{NaiTamYokPynMar-ijcai03} searches for Nash equilibrium solutions, \ie{}, each agent's policy being a {\em best response} to the other agents' policies.
It does so in the finite-horizon setting (relying on tree
representations for policies) by individually optimizing each agent's
policy one after the other, while fixing the other agent's policies,
until convergence, \ie{}, repeating the process until no improvement
is possible.
However, this algorithm faces two main drawbacks:
(1) the resulting Nash equilibria are local, not global, optima; and
(2) it addresses only finite-horizon problems.

In this paper, we address this second drawback and propose a way to solve infinite-horizon problems through a JESP approach, called \textit{\infJESP} (infinite-horizon JESP).
Its starting point is to rely not on policy trees, but on finite state controllers.
To that end, we provide a way to build the POMDP faced by some agent
when fixing the other agents' FSC policies. From there, we use a POMDP solver to find individual best-response
FSCs, and integrate this step into a JESP algorithmic scheme.

To be more precise, we extend and improve the JESP method on three aspects:
(1) tree representations are replaced by FSC representations to address infinite-horizon problems and to build each intermediate POMDP without having to consider distributions over the possible histories of the other agents (but only the internal nodes of their FSCs);
(2) current state-of-the-art POMDP solvers are used at each step (in this paper, we used SARSOP \cite{KurHsuLee-rss08}) to approximate the optimal value function of a POMDP, and an FSC is derived from the resulting approximation \cite{GrzPouYanHoe-ToCyb14};
(3) two novel (deterministic) heuristic initialization methods for JESP are provided, where individual FSCs are extracted from a joint policy obtained by solving a simpler Multi-agent POMDP (MPOMDP) \cite{Pynadath-jair02} in which all agents are controlled by a common entity that has access to all received observations.
By following these directions, we expect that using FSCs will help to
build policies which both (i) have compact representations and (ii) are easier to execute and to understand (as in
\cite{GrzPouYanHoe-ToCyb14}) than in classic JESP.

In Section~\ref{sec:RelatedWork}, we discuss related works about finite state controllers and existing Dec-POMDP solution methods.
Sec.~\ref{sec:Background} formally defines Dec-POMDPs and FSCs.
Sec.~\ref{sec:Contribution} (1) explains how to combine FSCs with a Dec-POMDP to generate the POMDP required at each \infJESP iteration;
(2) then describes the overall \infJESP algorithm dedicated to solving infinite-horizon Dec-POMDPs; and
(3) presents a heuristic initialization method for the JESP family of algorithms.
Finally, Sec.~\ref{sec:Experiments} presents empirical results and analyses  before  concluding. 

\ifextended{}{ Note: Complemental appendices appear in an extended
  version \cite{ictai21ext}. }

\section{Related Work}
\label{sec:RelatedWork}

\paragraph{Dec-POMDP techniques}

A first type of approach for solving Dec-POMDPs consists in
transforming the Dec-POMDP problem into a deterministic shortest path problem, as in Multi-Agent A* (MAA*)
\cite{SzeChaZil-uai05}, or even a Markov decision process with a state space of sufficient
statistics, as in FB-HSVI \cite{DibAmaBufCha-jair16} and PBVI-BB \cite{MacIsb-nips13}, which both rely on point-based solvers.
These approaches can give solutions as close to optimal as wanted, but the size of the state space for the corresponding problem blows up as the number of actions and observations grows, requiring a lot of computational resources.

A second type of approach consists in exploring the joint policy space by simultaneously optimizing the parametrized policies of all the agents.
For infinite-horizon Dec-POMDPs, these approaches represent each agent's (bounded-memory) policy compactly as an FSC, making it possible to directly search in a space of fixed-size FSCs.
\cite{AmaBerZil-jaamas10} proposed to directly optimize the FSC parameters through a non-linear programming (NLP) problem.
Using FSCs in the form of Mealy machines instead of Moore machines led to the improved MealyNLP solver \cite{AmaBonZil-aaai10}.
Other approaches \cite{KumZilTou-jair15, PajPel-ijcai11,PajPel-nips11}
address Dec-POMDPs as an inference problem consisting in estimating
the best parameters of the FSCs to maximize the probability of
generating rewards, leading to sub-optimal joint policies (due to
limitations of Expectation-Maximization).
In particular, PeriEM \cite{PajPel-nips11} works with periodic FSCs (as does Peri, a policy-graph improvement algorithm).
As policy-gradient algorithms for POMDPs (i) naturally extend to Dec-POMDPs \cite{TaoBaxWea-icml01}, and (ii) can serve to optimize the parameters of an FSC
\cite{MeuKimKaeCas-uai99,Aberdeen03PhD}, they would allow optimizing
multiple FSCs in a multi-agent setting as well.

A third type of approach uses heuristics to build a joint policy, as JESP (Joint Equilibrium-Based Search for Policies) \cite{NaiTamYokPynMar-ijcai03},
which is based on the observation that optimal joint policies are Nash equilibria.
JESP proposes to search for Nash equilibria by optimizing one agent's policy at a time, fixing the policies of the other agents, until no more improvement is possible.
At each iteration, a (single-agent) partially observable Markov decision process (POMDP) is faced, combining the dynamics of the Dec-POMDP and the known policies of the other agents.
The hidden state of this problem combines the environment state and other agents' internal state.
By turning the original problem into a sequence of single-agent
problems, JESP highly reduces the complexity of the resolution, but may fall in local optima.
Moreover, in JESP, POMDP solving was performed by either an exhaustive search or dynamic programming with policies represented by policy trees.
In this case, the other agents' internal states are their observation histories, so that the POMDP state space grows exponentially with the (necessarily finite) horizon.
We  adapt JESP to infinite-horizon problems by using an FSC policy representations.

Among Dec-POMDP approaches, Dec-BPI (Decentralized Bounded Policy Iteration) \cite{BerHanZil-ijcai05,BerAmaHanZil-jair09} is closely related to our proposal.
It uses a stochastic FSC representation with an added correlation device allowing the agents to share the same pseudo-random numbers, and proposes an approach similar to JESP by improving the parameters of one FSC node (or the correlation device) at a time through linear programming.
We instead propose to build at each iteration a totally new FSC for one of the agents.

\paragraph{FSC policy representation and evaluation}

This section discusses other related works, focusing on FSC representations for
infinite-horizon POMDP policies.

One of the main contributions in this context is policy iteration for POMDPs \cite{Hansen-nips97,Hansen-uai98}.
In these papers, \citeauthor{Hansen-nips97} proposes a policy iteration algorithm based on FSC policies (as well as a heuristic search).
Following the policy iteration algorithmic scheme, policy iteration for POMDPs relies on two steps:
(i) an evaluation step consisting in assessing the value of the current FSC policy, and
(ii) an improvement step upgrading the FSC policy based on this evaluation through the addition of new nodes and the pruning or replacement of dominated ones.
It must be noted that, while the policy improvement step can be very time-consuming, the policy evaluation step can be done efficiently by solving a system of linear equations \cite{Hansen-nips97}.
In the present article, we reuse the exact same technique to assess the value of the FSCs we obtain at each JESP iteration.

Note that policy gradients can also be applied to optimize a
parameterized (stochastic) FSC of fixed-size
\mbox{\cite{MeuKimKaeCas-uai99,Aberdeen03PhD}}.

More recently, \citeauthor{GrzPouYanHoe-ToCyb14} \cite{GrzPouYanHoe-ToCyb14} explained how to derive (and compress) an FSC from a value function. This allows to use state-of-the-art POMDP solvers to build a value function while expressing the final policy as an FSC. Here, we use similar techniques to
(i) build, at each iteration of \infJESP, an FSC
from a solution $\alpha$-vector set (Sec.~\ref{sec:FSCCompression}), and
(ii) equip the agents with initial policies (Sec.~\ref{sec:MPOMDPInitialization}).

\section{Background}
\label{sec:Background}

\subsection{Dec-POMDPs}

The problem of finding optimal collaborative behaviors for a group of
agents under stochastic dynamics and partial observability is
typically formalized as a {\em decentralized partially observable
  Markov decision process} (Dec-POMDP).

\begin{definition}
  A {\em Dec-POMDP} with $|\cI|$ agents is represented as a tuple $M \equiv \langle \cI, \cS, \cA, \Omega, T, O, R, b_0, H, \gamma \rangle$, where:
$\cI = \{1, \dots, |\cI|\}$ is a finite set of {\em agents};
$\cS$ is a finite set of {\em states};
$\cA = \bigtimes_i \cA^i$ is the finite set of joint actions, with $\cA^i$ the set of agent $i$'s {\em actions}; $\Omega = \bigtimes_i \Omega^i$ is the finite set of joint
  observations, with $\Omega^i$ the set of agent $i$'s {\em observations}; $T: \cS \times \cA \times \cS \to \reals$ is the {\em transition
    function}, with $T(s,a,s')$ the probability of transiting from $s$ to $s'$ if $a$ is
  performed;
$O: \cA \times \cS \times \Omega \to \reals$ is the
  {\em observation function}, with $O(a,s',o)$ the probability of observing $o$ if $a$ is performed and
  the next state is $s'$;
$R: \cS \times \cA \to \mathbb{R}$ is the {\em reward function},
  with $R(s,a)$ the immediate reward for executing $a$ in $s$;
$b_0$ is the {\em initial probability distribution} over states;
$H \in \mathbb{N} \cup \{\infty\}$ is the (possibly infinite) {\em time horizon};
$\gamma \in (0,1)$ is the {\em discount factor} applied to future rewards.
\end{definition}

An agent's $i$ action {\em policy} $\pi^i$ maps its possible
action-observation histories to actions.
The objective is then to find a joint policy
$\pi=\langle \pi^1, \dots, \pi^{|\cS|} \rangle$ that maximizes the expected discounted return from $b_0$:
\begin{align*}
  V^{\pi}_H(b_0)
  & \eqdef \mathbb{E}\left[ \sum_{t=0}^{H-1} \gamma^{-t} r(S_t, A_t) \mid S_0 \sim b_0, \pi \right].
\end{align*}

\subsection{POMDPs}

In this work, we rely on optimal Dec-POMDP solutions being {\em
  equilibria}, \ie, situations where each agent $i$ follows a {\em
  best-response} policy given the other agents' fixed policies noted
$\pi^{\neq i}$, what induces a single-agent Dec-POMDP, \ie, a POMDP.
In a POMDP, an optimal policy $\pi^*$ exists whose input is the belief
$b$, \ie, the probability distribution over states (or {\em belief}
$b$) given the current action-observation history.
For finite $h$, the optimal value function (which allows deriving
$\pi^*$) is recursively defined as:
\begin{align*}
  V^*_h(b)
  & = \max_a \left[r(b, a) + \gamma \sum_{o} Pr(o \mid b, a) V^*_{h-1}(b^{a,o}) \right].
\end{align*}
where (i) $r(b,a) = \sum_s b(s)\cdot r(s,a)$, (ii) $Pr(o | b, \pi(b))$ depends on the dynamics, and (iii) $b^{a,o}$ is the belief updated upon performing $a$ and
perceiving $o$.
For finite $h$, $V^*_h$ is known to be piece-wise linear and convex
(PWLC) in $b$.
For infinite $h$, $V^* (=V^*_\infty)$ can thus be
approximated by an upper envelope of hyperplanes---called
$\alpha$-vectors $\alpha \in \Gamma$.

\subsection{Finite State Controllers}
\label{sec:FSC_def}

In POMDPs as in Dec-POMDPs, solution policies can also be sought for in
the form of {\em finite state controllers} (FSC) (also called {\em
  policy graphs} \cite{MeuKimKaeCas-uai99}), \ie, automata whose
transitions from one internal state to the next depend on the received
observations and generate the actions to be performed.

\begin{definition}
  For some POMDP's sets $\cA$ and $\Omega$, an {\em FSC} is represented as a tuple
  $fsc \equiv \langle N, \eta, \psi \rangle$, where:
  \begin{itemize}
  \item $N$ is a finite set of (internal) nodes, with $n_0$ the start node; \item $\eta: N \times \Omega \times N \to \reals$ is the transition function between nodes of the FSC; $\eta(n,o,n')$ is the probability of moving from node $n$ to $n'$ if $o'$ is observed; the notation $n'=\eta(n,o)$ is also used when this transition is deterministic;
  \item $\psi: N \times \cA \to \reals$ is the action selection function of the FSC; $\psi(n,a)$ is the probability to choose action $a \in \cA$ in node $n$; the notation $a=\psi(n)$ is also used when this function is deterministic.
  \end{itemize}
\end{definition}

A deterministic FSC's value function (\ie,
$\eta$ and $\psi$ being both deterministic) is the solution of the
following system of linear equations, with one $\alpha$-vector per
node $n$ \cite{Hansen-nips97}:
\begin{align}
  \label{eqn:PolicyEvaluation}
  \alpha_{s}^{n} 
  & = R(s, a_n)+\gamma \sum_{s', o}T(s, a_n, s')O(a_n,s',o) \alpha_{s'}^{\eta(n, o)},
\end{align}
where $a_n\eqdef \psi(n)$.
Using the fixed point theorem, a solution can be found using an
iterative process typically stopped when the Bellman residual (the
largest change in value) is less than a threshold $\epsilon$, so that
the estimation error is less than $ \epsilon \over {1-\gamma}$.

\section{Infinite-horizon JESP}
\label{sec:Contribution}

\infJESP relies on a main local search, which is typically
randomly restarted  multiple times to converge to different local optima.
This local search, presented in Sec.~\ref{sec:infJESP},
relies on iteratively (i) defining a POMDP for each agent based on other agents' policies (Sec.~\ref{sec:BRPOMDP}), and (ii) solving it to extract and evaluate an associated FSC (Sec.~\ref{sec:FSCCompression}).
We present heuristic initializations for \infJESP in Sec.~\ref{sec:MPOMDPInitialization}.

\subsection{Main Algorithm} \label{sec:infJESP}

Each agent's policy is represented as a deterministic FSC (of which
there is a finite number).
To control the computational cost at each iteration, the size of
solution FSCs is bounded by a parameter $K \in \mathbb{N}^*$.
The local search thus starts with $|\cI|$ randomly generated $K$-FSCs in $\fsc$.
Then, it loops over the agents, each iteration attempting to improve
an agent $i$'s policy by finding (line~\ref{codes:solvePOMDP}) a $K$-FSC
$\fsc'_i$ that is a best response to the current (fixed)
FSCs $\fsc_{\neq i}$ of the $|\cI|-1$ other agents (denoted $\neq i$).
Line~\ref{codes:PolicyEval} relies on Eq.~\ref{eqn:PolicyEvaluation}
to evaluate $\langle \fsc'_i, \fsc_{\neq i}\rangle$ at
$b_0$, unless the POMDP solver at line~\ref{codes:solvePOMDP}
provides this information.
Then, if an improved solution has been found, $\fsc'_i$ replaces $\fsc_i$ in $\fsc$.
The process stops if the number of consecutive iterations without
improvement, $\nNI$, reaches $|\cI|$.

\begin{algorithm}
  \caption{$\infty$-horizon JESP}
  \label{alg:inftyHJESP}
\DontPrintSemicolon

  \SetInd{.3em}{.6em}
  \scalefont{.9}
  
  \SetKwFunction{LocalSearch}{{\bf LocalSearch}}
  \SetKwFunction{SolveToFSC}{{\bf Solve2FSC}}
  
  [Input:] $K$: FSC size  $\mid$ $\fsc$: initial FSCs

  \Fct{\LocalSearch{$K, \fsc \eqdef \langle \fsc_1, \dots, \fsc_{|\cI|} \rangle$ } }{
    $v_{bestL} \gets eval(\fsc)$ \;
    $\nNI \gets 0$   {\scriptsize\tcp*{\#(iterations w/o improvement)}} 
    $i\gets 1$  {\scriptsize\tcp*{Id of current agent}}
    \Repeat{$\nNI=|\cI|$}{
      $\fsc'_i \gets $\SolveToFSC{$\fsc_{\neq i}$, $K$}\; \label{codes:solvePOMDP}
      $v \gets eval(\fsc'_i, fsc_{\neq i})$\; \label{codes:PolicyEval}
      \eIf{$v > v_{bestL}$}{
        $fsc_i \gets \fsc'_i$\;
        $v_{bestL} \gets v$\;
        $\nNI \gets 0$\;
      }{ $\nNI \gets \nNI+1$\;
      }
$i \gets (i \mod |\cI|) + 1$\;
    }\Return{$\langle \fsc, v_{bestL} \rangle$}
  }

\end{algorithm}

\paragraph{Properties}

For agent $i$, each of its $K$ nodes is attached to an action in $\cA_i$, and each of its $K\times|\Omega|$ edges is attached to a node, so that the number of {\em deterministic} FSCs $|\FSC_i|$ is
at most
$|\cA_i|^K \cdot K^{K\times|\Omega|}$.\footnote{The
  exact number is smaller due to symmetries and because, in some FSCs,
  not all internal nodes are reachable.}
Thus, assuming an optimal POMDP solver leads to the following
properties.

\begin{proposition}
  \infJESP's local search converges in finitely many iterations to a
  Nash equilibrium.
\end{proposition}

\begin{proof}
  The search only accepts increasingly better solutions, so that the
  number of iterations (over all agents) is at most the number of
  possible solutions: $|\FSC| = \prod_i |\FSC_i|$.

  The search stops when each agent's FSC is a best response to the
  other agents' FSCs, \ie, in a Nash equilibrium.
\end{proof}

While  these equilibria are only local optima,
allowing for infinitely many random restarts guarantees converging to
a global optimum with probability $1$.
Of course, the set of Nash equilibria depends on the set of policies
at hand, thus on the parameter $K$ in our setting.
Increasing $K$ just allows for more policies, and thus possibly better
Nash equilibria.

In practice (see Sec.~\ref{sec:BRPOMDP}), we use a sub-optimal POMDP
solver in \SolveToFSC (line~\ref{codes:solvePOMDP}) in which FSC sizes
are implicitly bounded ($K$ is not actually used) because the solving time is bounded.
If this POMDP solver returns a solution $\fsc'_i$ worse than $\fsc_i$,
it will be ignored, so that monotonic improvements are preserved, and
the search still necessarily terminates in finite time.
Final solutions may be close to Nash equilibria if the POMDP solver returns
$\epsilon$-optimal solutions.

\subsection{Best-Response POMDP For Agent $i$}
\label{sec:BRPOMDP}

\paragraph{Finite-horizon Solution (JESP)}

In JESP, when reasoning about an agent $i$'s policy while considering that other agent's policies $\pi_{\neq i} $ are known and fixed, agent $i$ maintains a belief $b^{t}_{JESP}$ over its own state $e^t \eqdef \langle  s^t, \vec{\omega}_{\neq i}^{t} \rangle$, with $s^t$ the current state and $\vec{\omega}_{\neq i}^{t}$ the observation histories of other agents. 
This belief is a sufficient statistics for planning as it allows
predicting the system's evolution (including other agents) as well
as future expected rewards.

However, in infinite-horizon problems, the number of observation
histories $ \vec\omega_{\neq i}^{t}$ grows exponentially with time,
making for an infinite state space.
With FSC policies, agent $i$ can reason about other agents' internal nodes instead.

\paragraph{Infinite-horizon Solution}

Multiple best-response POMDPs can be designed (see \ifextended{Appendix~\extref{app:POMDPformalization} on page~\extpageref{app:POMDPformalization}}{\cite[Appendix~\extref{app:POMDPformalization}]{ictai21ext}}).
Here, the {\em extended} state $e^t \in \cE$ contains
(i) $s^t$, the current state of the Dec-POMDP problem, 
(ii) $n_{\neq i}^{t} \equiv \langle n_{1}^{t}, \dots, n_{i-1}^{t}, n_{i+1}^{t}, \dots, n_{n}^{t} \rangle$, the nodes of the $|\cI| -1$ other agents' (agents $\neq i$) FSCs at the current time step, and 
(iii) $o_{i}^{t} $, agent $i$'s current observation.
We thus have $\cE \eqdef \cS \times N_{\neq i} \times \Omega_{i}$.
Given an action $a_i^t$, the extended state
$e^t\equiv \langle s^t, n_{\neq i}^t, o_i^t\rangle$ evolves according
to the following steps (see Figure~\ref{BRModel}):
(1) each agent $j \neq i$ selects its action $a_j$ based on its
current node;
(2) $s^t$ transitions to $s^{t+1}$ according to $T$ under joint action
$a^t \equiv \langle a_i^t, a_{\neq i}^t\rangle$; and
the FSC nodes of agents $j$ (including $i$) evolve jointly (possibly
not independently) according to their observations $o_j^{t+1}$, which
depend on $s^{t+1}$ and the joint action $a^t$.

\begin{figure}\centering
     \includegraphics[width=1 \linewidth]{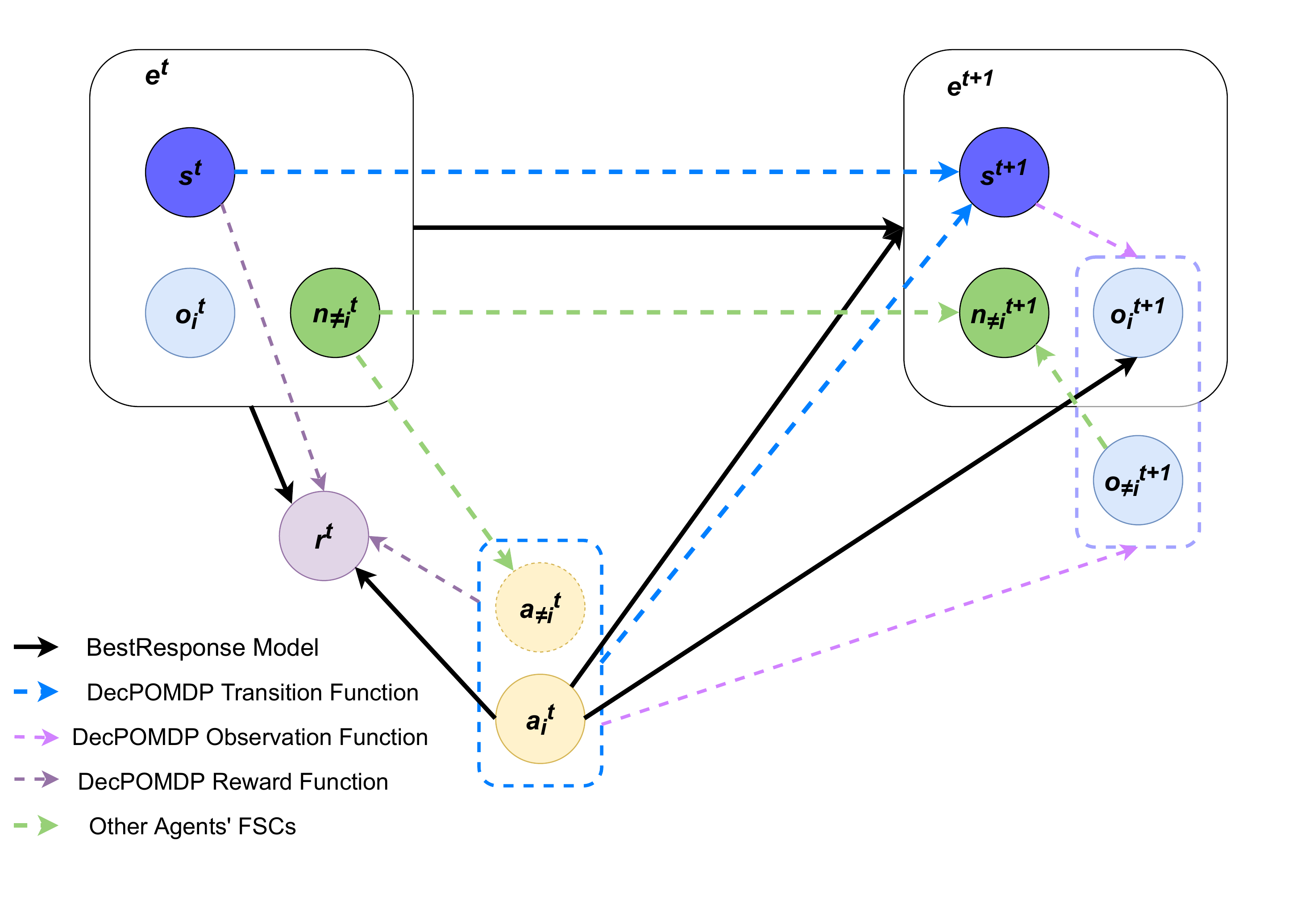}
     \caption{The decision-making structure for the best-response POMDP: The black arrows give a standard POMDP model; 
The purple dashed arrows show the transition function $T$ of the Dec-POMDP, \ie, the probability of next state $s^{t+1}$ given on current state $s^t$ and joint action $\langle a_{\neq i}^{t}, a^t_{i} \rangle$;
The blue dashed arrows show the observation function $O$ of the Dec-POMDP, \ie, the probability of joint observations $\langle o^{t+1}_{\neq i},o^{t+1}_{i} \rangle$ given the next state $s^{t+1}$ and joint action $\langle a_{\neq i}^{t}, a^t_{i} \rangle$;
The pink dashed arrows show the reward function;
The green arrows show the evolution of agents $\neq i$' FSC nodes.
 }
     \label{BRModel}
\end{figure}  

This design choice for $e^t$ induces a POMDP because of the following
properties: (i) it induces a Markov process controlled by the action (see
dynamics below); (ii) the observation $o^t_i$ depends on $a^t_i$ and $e^{t+1}$; and (iii) the reward depends on $e^t$ and $a^t_i$.
Indeed, deriving the transition, observation and reward functions for this POMDP (\cf App.~\extref{app:POMDPformalization} \ifextended{}{in \cite{ictai21ext}}) leads to:
\begin{align*}
  &  T_e(e^t, a^t_i, e^{t+1})  = Pr(e^{t+1}| e^t, a^t_i) \\
  & \quad = \sum_{o^{t+1}_{\neq i}}T(s^t, \langle \psi_{\neq i}(n_{\neq i}^{t}), a^t_{i} \rangle, s^{t+1}) 
  \cdot \eta_{\neq i}(n_{\neq i}^{t},o^{t+1}_{\neq i},n_{\neq i}^{t+1}) \\
  & \qquad \cdot O(s^{t+1}, \langle \psi_{\neq i}(n_{\neq i}^{t}), a^t_{i} \rangle, \langle o^{t+1}_{\neq i},o^{t+1}_{i} \rangle),
  \\ & O_e(a^t_i, e^{t+1}_i, o^{t+1}_i)
  = Pr( o^{t+1}_i | a^t_i, e^{t+1}_i) \\
  & \quad = Pr( o^{t+1}_i | a^t_i, \langle s^{t+1}, n^{t+1}_{\neq i}, \tilde o^{t+1}_i \rangle) = \mathbf{1}_{o^{t+1}_i = \tilde  o^{t+1}_i},
  \\
  & r_e(e^t, a^t_i) =  r(s^t, a^t_i,\psi_{\neq i}(n_{\neq i}^{t}) ),
\end{align*}
where $\eta_{\neq i}(n_{\neq i}^{t},o^{t+1}_{\neq i},n_{\neq i}^{t+1})  = \prod_{j\neq i} \eta(n_j^{t}, \tilde o_j^{t+1}, n_j^{t+1}) $ and $\psi_{\neq i}(n_{\neq i}^{t}) = a^t_{\neq i}$. 

Note that the full observability of state variables such as $o^t_i$
can be exploited \cite{OngPngHsuLee-rss09,AraThoBufCha-ICTAI10}, what
we leave to the POMDP solver.
Also, we make $\cE$ smaller through {\em state elimination}, \ie, only keeping extended states
reachable from the initial belief.

\subsection{Obtaining and Evaluating $\fsc_i$}
\label{sec:FSCCompression}

\SolveToFSC could rely on a solver that directly returns an FSC
\cite{Hansen-nips97}.
Instead, we choose to rely on a modern point-based approach (\eg,
PBVI, HSVI or SARSOP
\cite{PinGorThr-ijcai03,SmiSim-uai04,KurHsuLee-rss08}) returning an
$\alpha$-vector set $\Gamma$ that approximates the optimal value
function.

\paragraph*{Algorithm}

We use Algorithm~\ref{alg:BuildFSC}, which is similar to \mbox{\citeauthor{GrzPouYanHoe-ToCyb14}}'s
Alpha2FSC, to turn an $\alpha$-vector set $\Gamma$ into an FSC, where each node
$n\eqdef \langle \alpha, a, b, w \rangle$ contains an $\alpha$-vector $\alpha$ with its (sometimes omitted) action $a$, and a representative belief $b$ (weighted average of the encountered
beliefs mapped to this node) with its positive weight $w$ (which tries to capture the importance of
this node).
It first creates a start node $n_0$ from $b_0$ and $\Gamma$ with $w=1$
(line~\ref{alg_part:StartNode}), and adds it to both the new FSC ($N$)
and an open list ($G)$.
Then, each $n$ in $G$ is processed (in fifo order) with each
observation $o_i$ as follows ($0$-probability observations induce a
self-loop (line~\ref{alg_part:selfLoop})).
Lines~\ref{line|bupdate}--\ref{alg_part:ArgmaxAlphaVec} compute the
updated belief $\bb{a_i}{o_i}$, and the associated $w'$ and
$\alpha_i$.
The node $n' \in N$ that contains $\alpha_i$ is extracted if it exists
and $\bb{a_i}{o_i}$ is merged into its representative belief
(lines~\ref{alg_part:nodeExists}--\ref{alg_part:beliefMerge}), otherwise a new node $n'\eqdef \langle \alpha_i, \bb{a_i}{o_i}, w'\rangle$
is created (line~\ref{alg_part:CreateNewNode}) and added to both $N$
and $G$.
An edge $n\to n'$, labeled with $o_i$, is then added
(line~\ref{alg_part:LinkNodes}).

As in \citeauthor{GrzPouYanHoe-ToCyb14}'s work, self-loops are added when an impossible observation $Pr(o_i|b,a_i) = 0$ is received.
This may happen because, when building the FSC, each node $n_i$ is associated to a tuple $\langle \alpha_i, b_i \rangle$, and outgoing edges from $n_i$ will be created only for observations that have non-zero probability in $b_i$.
Yet, during execution $n_i$ may be reached while in a different belief $b'$  from which $a_i$ (the action attached to $\alpha_i$, thus $n_i$) may induce ``unexpected'' observations.\footnote{This will happen when some state $s$ has zero-probability in $b_i$ but not in $b'$ and can induce this observation when performing $a_i$.}
Adding self-loops is a way to equip the agent with a default strategy when such an unexpected observation occurs.

However, there are three differences in our method compared with \citeauthor{GrzPouYanHoe-ToCyb14}'s work:
First, in \cite{GrzPouYanHoe-ToCyb14}, each input  $\alpha$-vector comes with an associated beliefs, while we instead try to compute a belief representative of the reachable beliefs ``attached'' to the vector.
Second, in \cite{GrzPouYanHoe-ToCyb14}, each $\alpha$-vector induces an FSC node ($|N|=|\Gamma|$), while we only consider $\alpha$-vectors reachable by the algorithm from $b_0$.
Third, in our JESP setting, we have another reason for adding self-loops.
Indeed, each agent $i$'s FSC is obtained considering fixed agent $\neq i$'s FSCs;
yet, changing other agents' FSCs will induce a new POMDP from $i$'s point of view, potentially with new possible observations when in certain nodes of $\fsc_i$.

\begin{algorithm}
  \caption{Extract FSC $\langle N, \eta, \psi \rangle$ from set $\Gamma$}
  \label{alg:BuildFSC}
\DontPrintSemicolon

  \SetInd{.3em}{.6em}
  \scalefont{.9}

{[Input:]} $\Gamma$: $\alpha$-vector set 

  Start node $n_{0} \gets node( \langle \argmax_{\alpha \in \Gamma} ( \alpha \cdot b_{0}), b_{0}, 1 \rangle )$ \label{alg_part:StartNode} \;
  $N \gets \{n_0\}$ \; $G.pushback(n_0)$ \; \While{$|G| > 0$}{
	$n \gets G.popfront()$\;
	$\langle b, a_i, w \rangle \gets \langle n.b, \psi(n), n.w \rangle$\;
	\ForAll{$o_{i} \in \Omega_{i}$}{
          \eIf{$Pr(o_i|b, a_i) > 0$  \label{alg_part:CheckPoba}  }{
                        $\bb{a_i}{o_i} \gets beliefUpdate(b,a_i,o_i)$ \label{line|bupdate}\;
                        $w' \gets w \cdot Pr(o_i|b, a_i)$ \label{line|bupdate|weight}\;
			$\alpha_{i} \gets  argmax_{\alpha \in \Gamma} (\alpha \cdot \bb{a_i}{o_i})$   \label{alg_part:ArgmaxAlphaVec}  \;
			\eIf{$\alpha_{i} \notin N$  \label{alg_part:CheckAlphaVec} }{
                          $n' \gets node( \langle \alpha_i, \bb{a_i}{o_i}, w' \rangle )$  \label{alg_part:CreateNewNode}  \;
                          $N \gets N \cup \{n'\}$ \;
                          $G.pushback(n')$\;
                        }{
                          $n' \gets N(\alpha_i)$ \label{alg_part:nodeExists} \;
$n'.b \gets  \frac{n'.w}{n'.w + w'} \cdot n'.b + \frac{w'}{n'.w + w'}  \cdot\bb{a_i}{o_i}$\;
                          $n'.w \gets n'.w + w'$ \label{alg_part:beliefMerge}
                        }
                      }{
                        $n' \gets n$ \label{alg_part:selfLoop} \;
                      }
                      $\eta(n, o_i) \gets n'$  \label{alg_part:LinkNodes} \;
                    }
}

 \Return{$\langle N, \eta, \psi \rangle$}
\end{algorithm}

Note that this algorithm does not bound the resulting number of internal nodes.
Instead, we rely on (i) SARSOP returning finitely many $\alpha$-vectors, and (ii) the FSC extraction producing (significantly) less than $|\Gamma|$
internal nodes.

Once $\fsc_i$ is obtained, the next step is to evaluate all agents' policies in the Dec-POMDP.
To that end, it is sufficient to evaluate $\fsc_i$ in the corresponding best-response POMDP, since the latter combines in a single model the Dec-POMDP and the FSCs for agents $\neq i$.
Here we used the FSC evaluation technique described in Sec.~\ref{sec:FSC_def}.

\subsection{MPOMDP Initializations} 
\label{sec:MPOMDPInitialization}

As previously mentioned, \infJESP's  initialization  is important.
Random initializations may often lead to poor local optima.
Thus, we want to investigate if some non-random heuristic initializations allow to find good solutions quickly and reliably.
Our method relies on solving an MPOMDP \cite{Pynadath-jair02}, \ie, assuming that observations are public.
We extract individual FSCs from this MPOMDP policy as detailed below, and use them to initialize \infJESP.

\paragraph{MPOMDP-based Stochastic Initial FSC (M-S)}
Algorithm \ref{alg:MpomdpSampling} describes how to extract a policy $\fsc_i$ for agent $i$ from an MPOMDP policy.
This approach is similar to Algorithm~\ref{alg:BuildFSC}, the main
difference being that agent $\neq i$'s observations and actions should
also be considered to compute a next belief.
Yet, this is not actually feasible since agent $i$ is not aware of
them.
To address this issue, given agent $i$, let us consider a current node $n = \langle \alpha, b \rangle$ and an individual observation $o_i$ (\ie, with non-zero occurrence
probability). The MPOMDP solution specifies a joint action
$a = \langle a_i, a_{\neq i}\rangle$ at $b$, $a_i$ being the action
attached to $\alpha$.
We arbitrarily assume that (i) each possible stochastic transition of the FSC (from $n$ and coming with
observation $o_i$) corresponds to a possible joint observation
$o_{\neq i}$ of the other agents, (ii) such a transition occurs with probability
$Pr(o_{\neq i} \mid b, a, o_i)$, and (iii) it ``leads'' to a new belief
$\bb{a}{\langle o_i, o_{\neq i}\rangle}$ and thus a node $n'$ attached
to the dominating MPOMDP $\alpha$-vector ($\alpha_i$, \cf
line~\ref{alg_part:FindArgmaxAlphaWithJointBelief}) at this point.
As only one FSC node should correspond to a given $\alpha$-vector, a
new $n'$ attached to $\alpha_i$ is created only if necessary
(lines~\ref{line:testAlphaInN}, \ref{line:createNPrime} and
\ref{line:findNPrime}).
As multiple joint observations $o_{\neq i}$ may lead to the same $n'$,
the corresponding transition probabilities are cumulated in
$\eta(n, o_i, n')$ (line~\ref{line:addObsProb}).
$0$-probability joint observations $o_{\neq i}$ given $b$, $a$ and
$o_i$ are ignored.
$0$-probability individual observations $o_i$ given $b$ and $a$ induce
the creation of a self-loop (line~\ref{line:selfLoop}).
Note that the resulting FSCs are stochastic, and some of them may
never be improved on,
so that, in this case, the solution returned by \infJESP may contain stochastic
FSCs.

\paragraph{MPOMDP-based Deterministic Initial FSC (M-D)}

A very simple variant of this method is to assume that the only
possible transition from $n$ under $o_i$ corresponds to the most
probable joint observation $o_{\neq i}$ of the other agents. Node transitions are then deterministic ($\eta(n, o_i, n') = 1$).

\medskip

Notes: (1) As both (i) these heuristic initializations and (ii) \infJESP's
local search are deterministic, using a deterministic procedure to
derive best-response FSCs induces a deterministic algorithm for which
restarts are useless.
(2) These heuristic initializations can be also adapted to JESP's
finite-horizon setting by using policy trees.

\begin{algorithm}
  \caption{Extract $FSC_i$ from MPOMDP set $\Gamma$} 
  \label{alg:MpomdpSampling}
  \DontPrintSemicolon

  \SetInd{.3em}{.6em}
  \scalefont{.9}

{[Input:]} $i$: agent $\mid$ $\Gamma$: MPOMDP $\alpha$-vector set

  Start node $n_{0} \gets node( \langle argmax_{\alpha \in \Gamma} \alpha \cdot b_{0}, b_0, 1 \rangle)$ \;
  $ N \gets \{n_{0}\}$ \;
  $ G.pushback(n_{0})$ \;
\While{$|G| > 0$}{
	$n \gets G.popfront()$\;
	$\langle b, a, w \rangle \gets \langle n.b, \psi(n), n.w \rangle$\;
        \ForAll{ $o_{i} \in \Omega_{i}$ \label{alg_part:FindAllObsForAgentI} }{
          \eIf{$Pr(o_i | b, a) > 0$}{
            \ForAll{ $o_{\neq i} \in \Omega_{\neq i}$}{
              \If{$Pr(o_{\neq i}|o_i, b, a) > 0$ \label{alg_part:CheckPobsOthers} }{
                $\bb{a_i}{\langle o_i, o_{\neq i} \rangle} \gets beliefUpdate(b,a_i,\langle o_i, o_{\neq i} \rangle)$ \;
                $w' \gets w \cdot Pr(\langle o_i, o_{\neq i} \rangle|b, a_i)$ \;
                $\alpha_{i} \gets  argmax_{\alpha}(\bb{a}{\langle o_i, o_{\neq i}\rangle}, \Gamma)$ \label{alg_part:FindArgmaxAlphaWithJointBelief}  \;
                \eIf{$\alpha_{i} \notin N$ \label{line:testAlphaInN}}{
                  $n' \gets node(\langle \alpha_{i}, \bb{a}{}, w' \rangle)$  \label{line:createNPrime}\;
                  $N \gets N \cup \{n'\}$\;
                  $G.pushback(n')$}{
                  $n' \gets N(\alpha_i)$  \label{line:findNPrime} \;
$n'.b \gets  \frac{n'.w}{n'.w + w'} \cdot n'.b + \frac{w'}{n'.w + w'}  \cdot\bb{a_i}{o_i}$\;
                  $n'.w \gets n'.w + w'$
                }
                $\eta(n, o_i, n') \gets \eta(n, o_i, n') + Pr(o_{\neq i}|o_i, b, a) $ \label{line:addObsProb}}
            }
          }{
              $\eta(n, o_i, n) \gets 1$ \label{line:selfLoop}
          }
}
      }
 \Return{$\langle N, \eta, \psi \rangle$}
\end{algorithm}

\section{Experiments}
\label{sec:Experiments}

In this section, we evaluate variants of  \infJESP and other state-of-the-art solvers on various Dec-POMDP benchmarks.

\subsection{Comparison method}
\label{sec:cmp_method}

The standard Dec-POMDP benchmark problems we use are: DecTiger, Recycling, Meeting in a 3$\times$3 grid, Box Pushing, and
Mars Rover (\cf \url{http://masplan.org/problem}).

We compare the different variants of \infJESP with state-of-the-art Dec-POMDP algorithms, namely: FB-HSVI \cite{DibAmaBufCha-jair16}, Peri \cite{PajPel-nips11}, PeriEM \cite{PajPel-nips11}, PBVI-BB \cite{MacIsb-nips13} and MealyNLP \cite{AmaBonZil-aaai10}.
We ignore JESP as it only handles finite horizons.
Dec-BPI is compared separately because we can only estimate values from empirical graphs on some benchmark problems \cite{BerAmaHanZil-jair09}.

\subsection{Algorithm settings}
\label{sec:algo_settings}

SARSOP \cite{KurHsuLee-rss08} is our POMDP solver, with a $0.001$
Bellman residual (also for FSC evaluation) and a $5$\,s timeout.

We have tested 4 variants of \infJESP.
IJ(M-D) and IJ(M-S) are \infJESP
initialized with the M-D and M-S heuristics (Sec.~\ref{sec:MPOMDPInitialization}) and without restarts (because they behave deterministically).
IJ(R-$1$) and IJ(R-$100$) are \infJESP with random
initializations (at most 5 nodes per FSC) and, respectively, 1
(re)start and 100 restarts.
Each restart has a timeout of $7200$\,s.
Due to the limited time, we did not perform random restart tests on the Box-Pushing and Mars Rover domains.
IJ(R-$x_y$) denotes $y$ runs of IJ(R-$x$) used to compute statistics.

The experiments with \infJESP were conducted on a laptop with an i5-1.6\,GHz CPU and 8\,GB RAM.
The source code is available at \url{https://gitlab.inria.fr/ANR_FCW/InfJESP} .

\subsection{Results}
\label{sec:results}

\subsubsection{Comparison with state-of-the-art algorithms}

\begin{table}
  \caption{Comparison of different algorithms in terms of final FSC size, number of iterations, time and value, on five infinite-horizon benchmark problems with $\gamma = 0.9$ for all domains. \newline
R-$R_N$ = $N$ runs with $R$ random restarts each.
M-S and M-D = deterministic initializations with stochastic extracted FSCs and deterministic extracted FSCs from a MPOMDP solution. 
}
  \label{Table:BenchmarksResults}
  \resizebox{\linewidth}{!}{

\def\InfJesp{IJ}
\def\Rand{R}

\newcommand{\stdv}[1]{{\scriptstyle \pm #1}}

\begin{tabular}{ScSSS}
  \toprule
  {Alg.} & {FSC size} & {\#Iterations} & {Time (s)} & {Value}\\
  \midrule
  \multicolumn{5}{c}{ {DecTiger} {($|\cI|=2, |\cS|=2, |\cA^{i}|=3, |\cZ^{i}|=2 $)}} \\ 
  \midrule
  {FB-HSVI} & & & 153.7 & 13.448 \\
  {PBVI-BB} & & &  & 13.448 \\
  {Peri} & & & 220 & 13.45  \\
  \textbf{\InfJesp(\Rand-$100_5$)} & {$9 \stdv{3} \times 8 \stdv{4}$} & 33{$\!\!\stdv{1}$} & 12417.6 & 13.44 \\
  \textbf{\InfJesp(M-D)} & {$6 \times 6$} & 36 & 201 & 13.44 \\
  \textbf{\InfJesp(M-S)} & {$6 \times 6$} & 27 & 213  & 13.44 \\
  {PeriEM} & & & 6450 & 9.42 \\
  {MealyNLP} & & & 29 & -1.49 \\ 
  \textbf{\InfJesp(\Rand-$1_{500}$)} & {$29 \stdv{55} \times 32 \stdv{54}$} & 21{$\!\!\stdv{14}$} & 124.18  & -35.47 \\
  \midrule
  \multicolumn{5}{c}{{Recycling} {($|\cI|=2, |\cS|=4, |\cA^{i}|=3, |\cZ^{i}|=2 $)}} \\
  \midrule
  {FB-HSVI} & & & 2.6 & 31.929 \\
  {PBVI-BB} & & &  & 31.929 \\
  {MealyNLP} & &  & 0 & 31.928 \\
  {Peri} & & & 77 & 31.84  \\
  {PeriEM} & & & 272 & 31.80 \\
  \textbf{\InfJesp(\Rand-$100_{10}$)} & {$2 \stdv{0} \times 2 \stdv{0}$} &  3{$\!\!\stdv{0}$} & 3.1 & 31.62 \\
  \textbf{\InfJesp(\Rand-$1_{1000}$)} & {$3 \stdv{2} \times 3 \stdv{2}$} & 3{$\!\!\stdv{0}$} & 0.03 & 30.60 \\
  \textbf{\InfJesp(M-S)} & {$ 8 \times 8 $} & 3 & 0 & 26.57 \\ 
  \textbf{\InfJesp(M-D)} & {$ 4 \times 4 $} & 3 & 0 & 25.65 \\
  \midrule
  \multicolumn{5}{c}{{Grid3*3} {($|\cI|=2, |\cS|=81, |\cA^{i}|=5, |\cZ^{i}|=9 $)}} \\
  \midrule
  \textbf{\InfJesp(M-D)} & {$ 8 \times 10 $} & 3 & 2 & 5.81 \\
  \textbf{\InfJesp(M-S)} & {$ 9 \times 17 $} & 3 & 9 & 5.81 \\
  \textbf{\InfJesp(\Rand-$100_5$)} &  {$\ 13 \stdv{6} \times 9 \stdv{0}$} &  4{$\!\!\stdv{0}$} & 413.8 & 5.81 \\
  {FB-HSVI} & & & 67 & 5.802 \\
  \textbf{\InfJesp(\Rand-$1_{500}$)} & {$11 \stdv{18} \times 11 \stdv{18}$} & 3{$\!\!\stdv{1}$}  & 4.14 & 5.79 \\
  {Peri} & &  & 9714 & 4.64 \\
  \midrule
  \multicolumn{5}{c}{{Box-pushing} {($|\cI|=2, |\cS|=100, |\cA^{i}|=4, |\cZ^{i}|=5 $)}} \\
  \midrule
  {FB-HSVI} & &  & 1715.1 & 224.43 \\
  {PBVI-BB} & &  &  & 224.12 \\
  \textbf{\InfJesp(M-S)} & {$ 250 \times 408 $} & 6 &  963 & 220.25 \\
  \textbf{\InfJesp(M-D)} & {$ \ 274 \times 342 $} & 8 &  1436 & 203.41 \\
  {Peri} & & & 5675 & 148.65 \\
  {MealyNLP} &  &  & 774 & 143.14 \\ 
  {PeriEM} & & & 7164 & 106.65 \\
  \midrule
  \multicolumn{5}{c}{{Mars Rover} {($|\cI|=2, |\cS|=256, |\cA^{i}|=6, |\cZ^{i}|=8 $)}} \\
  \midrule
  {FB-HSVI} & & & 74.31 & 26.94 \\
  \textbf{\InfJesp(M-D)} & {$ 125 \times 183 $} & 6 & 122 & 26.91 \\
  \textbf{\InfJesp(M-S)} & {$ 84 \times 64 $} & 5 & 66 & 24.17 \\
  {Peri} & & & 6088 & 24.13 \\
  {MealyNLP} & &  & 396 & 19.67 \\ 
  {PeriEM} & & & 7132 & 18.13 \\ 
  \bottomrule
\end{tabular}

   }
\end{table}

Table~\ref{Table:BenchmarksResults} presents the results for the 5 problems. For IJ(R-$x_y$), in each run we kept the highest value among the restarts, and then computed the average of this value over the various runs.
The columns provide:
(\textit{Alg.}) the different algorithms at hand,
(\textit{N}) the final FSCs' sizes (for \infJESP[s]),
(\textit{I}) the number of iterations required to converge (for \infJESP[s]),
(\textit{Time}) the running time,
(\textit{Value}) the final value (lower bounds for \infJESP[s], 
the true value being at most 0.01 higher).

In terms of final value achieved, \infJESP variants find good to very
good solutions in most cases, often very close to FB-HSVI's
near-optimal solutions.
Comparing with estimated values obtained by Dec-BPI on the DecTiger, Grid and Box-Pushing problems (using the figures in  \cite[p.~123]{BerAmaHanZil-jair09}), \infJESP outperforms Dec-BPI in these three problems.
Note also that Dec-BPI's results highly depend on the size of the considered FSCs and that Dec-BPI faces difficulties with large  FSCs, \eg,  in the DecTiger problem \cite{BerAmaHanZil-jair09}.

Regarding the solving time, \infJESP variants have different behaviors depending on the problem at hand.
For example, in DecTiger, IJ(M-S) and IJ(M-D) have almost the same solving time, and both of them require more time than  IJ(R-$1_{500}$).
However, in other problems, IJ(M-D) and IJ(M-S) may require very different solving times.
Overall, IJ(M-D) and IJ(M-S) can give good solutions within an acceptable amount of time.

\subsubsection{Study of \infJESP}

\def\scaleg{1} \def\scalem{1}
\begin{figure}\centering
       \includegraphics[width=\scaleg\linewidth]{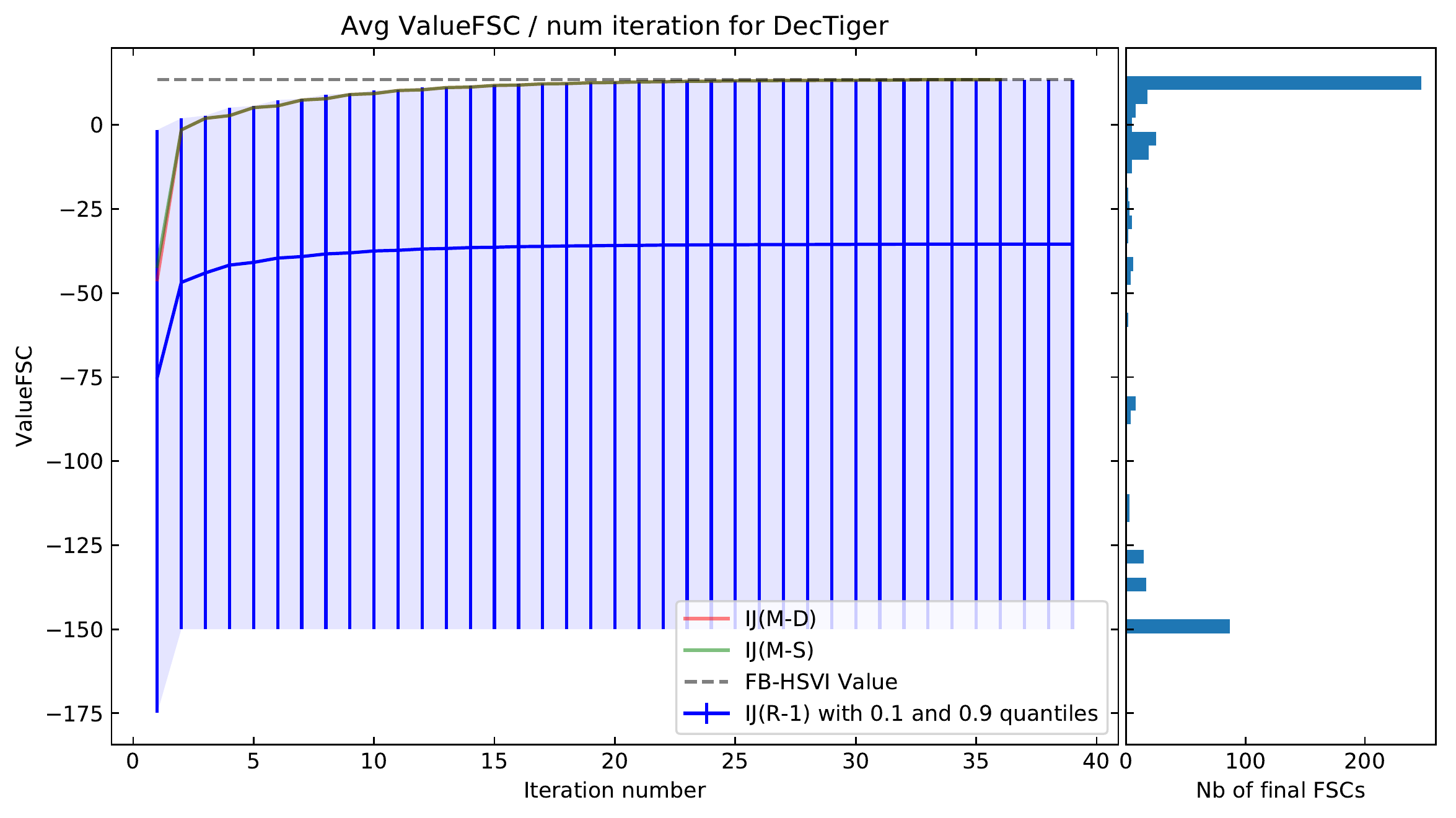}

       \includegraphics[width=\scaleg\linewidth]{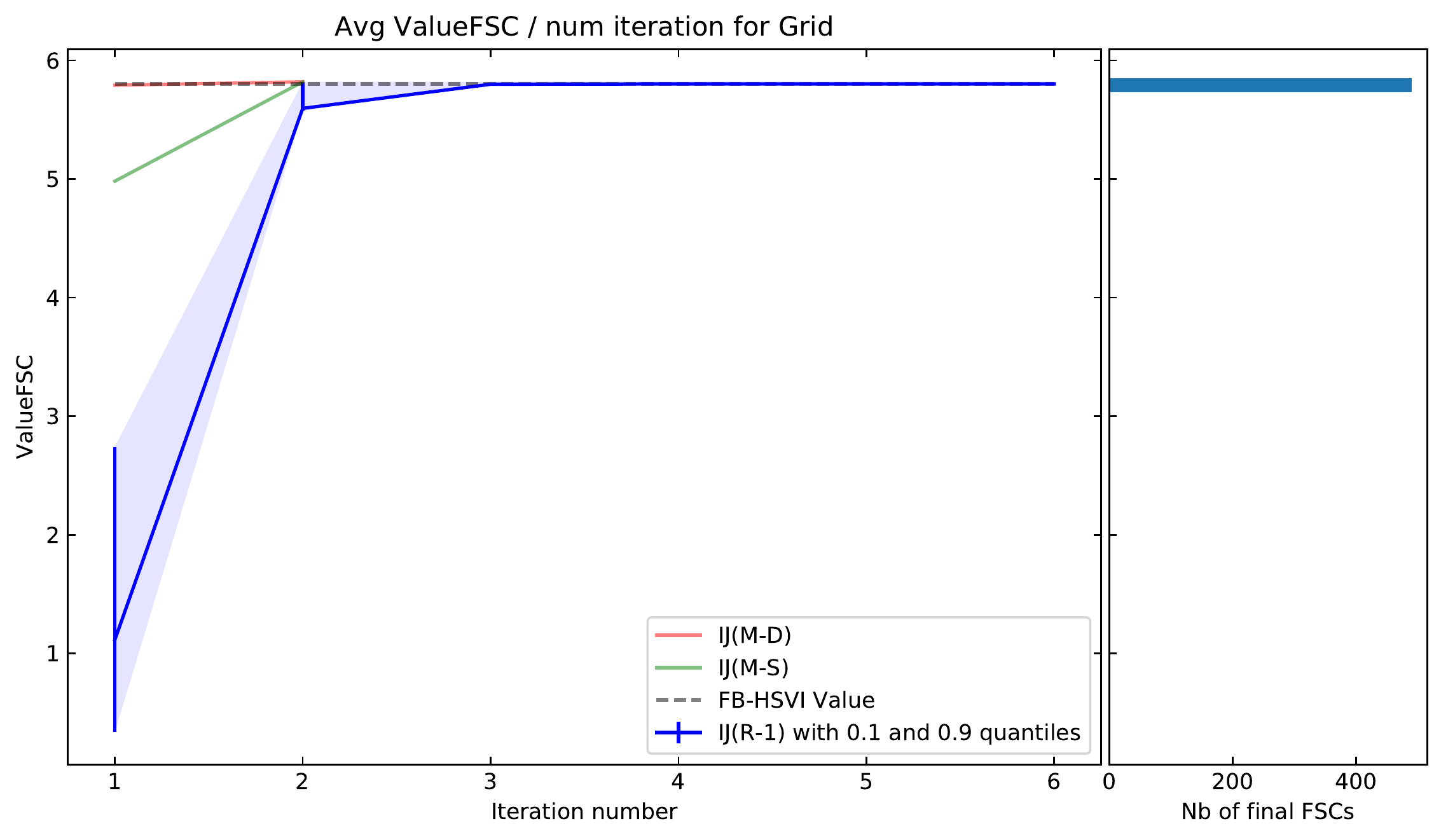}

       \includegraphics[width=\scaleg\linewidth]{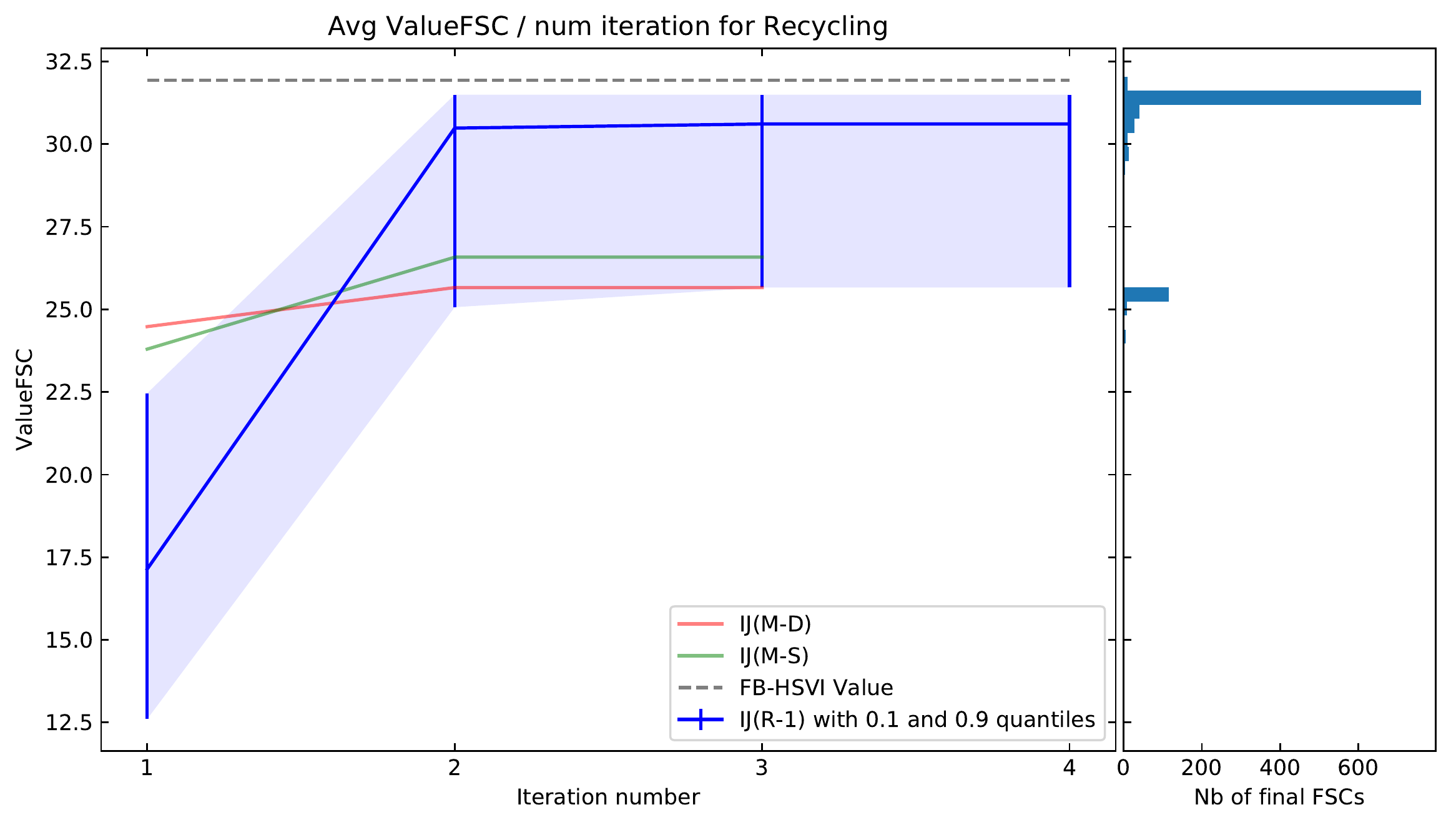}

       \caption{Values of the joint policy for the DecTiger, Grid and
         Recycling problems (from top to bottom).  The left part of each figure presents the evolution (during a
         run) of the value of the joint policy at each iteration of: IJ(R-$1$) (avg + 10th and 90th percentiles in blue), and the deterministic IJ(M-D) (in red) and IJ(M-S) (in green).
The dashed line represents FB-HSVI's final value.
The right part presents the value distribution after convergence of IJ(R-$1$). 
     }
     \label{Figure:ValueIterationAndFinal}
\end{figure}

Fig.~\ref{Figure:ValueIterationAndFinal} (right) presents the distribution over final values of IJ(R-$1$).
In the three problems at hand, most runs end with (near-) globally optimal values, despite initial FSCs limited to 5 nodes.
Other local optima turn out to be rarely obtained, except in DecTiger.
These distributions show that few restarts are needed to reach a global optimum with high probability.

Fig.~\ref{Figure:ValueIterationAndFinal} (left) presents the evolution of the FSC values during a JESP run as a function of the iteration number for
IJ(R-$1$) (avg + 10th and 90th percentiles in blue), IJ(M-D) (in green) and IJ(M-S) (in red).
For random initialization, at each iteration the average is computed
over all runs, even if they have already converged.
This figure first shows that (i) \infJESP's solution quality
monotonically increases during a run, and (ii) most runs converge to
local optima in few iterations.
It also illustrates that MPOMDP initializations are rather good heuristics compared to random ones but, as expected, do not always lead to a global optimum (\cf the Recycling problem).

\subsubsection{Complementary results}
The experiments also showed (\cf detailed results in
\ifextended{App.~\extref{app:CompleteResults}}{\cite[App.~\extref{app:CompleteResults}]{ictai21ext}})
that IJ(R-$1$) exhibits different behaviors depending on the problem
at hand.
For instance, removing unreachable extended states
(Sec.~\ref{sec:BRPOMDP}) divides their number on average by 1 in DecTiger (which is expected since, in this case, any observation
can be received from any state), 50 in Grid3*3, and 5 in Recycling.
We also observed that final FSCs with the highest values are not obtained with the more iterations and do not contain the more nodes.
Small FSCs seem sufficient to obtain good solutions, opening an interesting research direction to combine \infJESP with FSC compression.

\section{Conclusion}
\label{sec:Conclusion}

We proposed a new infinite-horizon Dec-POMDP solver, \infJESP, which
is based on JESP, but using FSC representations for
policies instead of trees.
FSCs allow not only to handle infinite horizons, but also to 
possibly derive compact policies.
Any restart of the ideal algorithm provably converges to a Nash
optimum, and, despite the existence of local optima, experiments show
frequent convergence close to global optima in five standard benchmark
problems.
One ingredient, the derivation of a POMDP from $|\cI|-1$ fixed
FSCs, could also be useful in other settings where other agents'
behaviors are defined independently, as games \cite{Oliehoek-MCLBN05}
or Human-Robot interactions \cite{BesAarBaRuDragan-2017}.

We also provided a method to extract individual policies (FSCs) from
an MPOMDP solution $\alpha$-vector set $\Gamma$ to initialize
\infJESP.
This approach can be easily adapted to JESP for the finite-horizon setting. 
Empirical results showed that this initialization method could, in some cases, reach good solutions with a value higher than the average answer of \infJESP with random initialization.
However, this approach does not always work.
In the Recycling problem, it is worse than the average \infJESP value with random initializations.
How to derive (possibly randomly) better heuristic initializations from MPOMDP policies remains an open question.

Future work includes experimenting with (i) bounded FSC sizes (to enforce compactness of policies); (ii) different best-response POMDP formalizations; (iii) different timeouts for SARSOP; (iv) parallelized restarts.

\ifdefined\extended
\appendices

\section{Notes about the Candidate POMDP Formalizations}
\label{app:POMDPformalization}

We here look in more details at how, given the Dec-POMDP and FSCs for all
agents but $i$, one can derive a valid best-response POMDP from agent
$i$'s point of view.
First, note that, in the POMDP formalization:
\begin{itemize}
\item we have {\em no choice} for the {observation} ($o^t_i$) and for
  the {action} ($a^t_i$), which are pre-requisites;
\item the {\em only choice} that can be made is that of the variables
  put in the {(extended) state} $e^t_i$;
\item the {transition, observation and reward functions} are
  {\em consequences} of this choice.
\end{itemize}

Also, note that, in a POMDP formalization, the state and observation
variables need be distinguished.
One cannot write that $X$ is an observed state variable (contrary to
the MOMDP formalism).
In the following, it is thus important to distinguish the different
types of random variables.
In particular,
\begin{itemize}
\item $O^t_i$ always denotes (of course) the {\em observation
    variable}, yet
\item $\tilde O^t_i$ denotes either an {\em intermediate
    variable}\footnote{Such a variable does not usually appear in the
    POMDP formalism, but is required here to compute the transition
    and observation functions based on the Dec-POMDP model and the
    FSCs.} or a {\em state variable} (which, in both cases, is
  strongly dependent on observation variable $O^i_t$ since their
  values are always equal).
\end{itemize}

The main issue when designing our BR POMDP is to verify that the
dependencies in the transition, observation, and reward functions are
the appropriate ones (see Fig.~\ref{fig:BRpomdp} (top left)).
In the following paragraphs, we consider 3 choices for the {\em
  extended state} of the POMDP, check whether they indeed induce valid
POMDPs, and derive the induced transition, observation and reward
functions when appropriate.

\paragraph{\uline{$e^t=\langle s^t, n^t_{\neq i}\rangle$ ?} }
To show that $e^t = \langle s^t, n^t_{\neq i} \rangle$ does not induce
a proper POMDP, let us write the transition function:
\begin{align*}
  & T_e(e^t, a^t_i, e^{t+1}) 
  \eqdef Pr(e^{t+1}| e^t, a^t_i) \\
  & = Pr(s^{t+1},n_{\neq i}^{t+1}| s^{t}, n_{\neq i}^{t},a^t_{i}) \\
  & = \sum_{\tilde o^{t+1}} Pr(s^{t+1},n_{\neq i}^{t+1}, \tilde o^{t+1}| s^{t}, n_{\neq i}^{t}, a^t_{i}) \\
  & = \sum_{\tilde o^{t+1}}
  Pr(n_{\neq i}^{t+1}| s^{t}, n_{\neq i}^{t}, a^t_{i}, s^{t+1}, \tilde o^{t+1}) \\
  & \quad \cdot Pr(s^{t+1}, \tilde o^{t+1}| s^{t}, n_{\neq i}^{t}, a^t_{i}) \\
  & = \sum_{\tilde o^{t+1}}
  Pr(n_{\neq i}^{t+1}| s^{t}, n_{\neq i}^{t}, a^t_{i}, s^{t+1}, \tilde o^{t+1}) \\
  & \quad \cdot Pr(\tilde o^{t+1}| s^{t}, n_{\neq i}^{t}, a^t_{i}, s^{t+1}) 
  \cdot Pr(s^{t+1}| s^{t}, n_{\neq i}^{t}, a^t_{i}) \\
  & = \sum_{\tilde o^{t+1}}
  \left( \prod_{j\neq i} \eta(n_j^t,\tilde o_j^{t+1}, n_j^{t+1}) \right) \\
  &  \cdot O(\langle \psi_{\neq i}(n_{\neq i}^{t}), a^t_{i} \rangle, s^{t+1}, \tilde o^{t+1})  
  \cdot T(s^{t}, \langle \psi_{\neq i}(n_{\neq i}^{t}), a^t_{i} \rangle, s^{t+1}),
\end{align*}
where $\tilde O^{t+1}$ is a temporary variable, not a state or an
observation variable.
Here, as illustrated by Fig.~\ref{fig:BRpomdp} (top right), the
issue is that the actual observation variable $O_i^{t+1}$ is not independent on
$E^t$ given $E^{t+1}$ and $A^t$.

\paragraph{\uline{$e^t = \langle s^t, n^t_{\neq i}, \tilde o^t_i \rangle$ ?} }
We correct the first attempt by adding a {\em state} variable
$\tilde O^t_i$, thus defining
$e^t = \langle s^t, n^t_{\neq i}, \tilde o^t_i\rangle$, hence:
\begin{align*}
  & T_e(e^t, a^t_i, e^{t+1}) 
  \eqdef Pr(e^{t+1}| e^t, a^t_i) \\
  & = Pr(s^{t+1}, n_{\neq i}^{t+1}, \tilde o_{i}^{t+1}| s^{t}, n_{\neq i}^{t}, \tilde o_{i}^{t}, a^t_{i}) \\
  & = \sum_{o^{t+1}_{\neq i}} Pr(s^{t+1},n_{\neq i}^{t+1}, o^{t+1}_{i-1}, \tilde o_i^t| s^{t}, n_{\neq i}^{t}, \tilde o_i^{t}, a^t_{i}) \\
  & = \sum_{o^{t+1}_{\neq i}}
  Pr(n_{\neq i}^{t+1}| s^{t}, n_{\neq i}^{t}, \tilde o_i^{t}, a^t_{i}, s^{t+1}, o^{t+1}) \\
  & \quad \cdot Pr(s^{t+1}, o^{t+1}_{\neq i}, \tilde o_i^{t+1}| s^{t}, n_{\neq i}^{t}, \tilde o_i^{t}, a^t_{i}) \\
  & = \sum_{o^{t+1}_{\neq i}}
  Pr(n_{\neq i}^{t+1}| s^{t}, n_{\neq i}^{t}, \tilde o_i^{t}, a^t_{i}, s^{t+1}, o^{t+1}) \\
  &  \cdot Pr(o^{t+1}_{\neq i}, \tilde o_i^{t+1}| s^{t}, n_{\neq i}^{t}, \tilde o_i^{t}, a^t_{i}, s^{t+1}) \cdot Pr(s^{t+1}| s^{t}, n_{\neq i}^{t}, \tilde o_i^{t}, a^t_{i}) \\
  & = \sum_{o^{t+1}_{\neq i}}
  \Big( \prod_{j\neq i} \eta(n_j^t, o_j^{t+1}, n_j^{t+1}) \Big) \\
  &  \cdot O(\langle \psi_{\neq i}(n_{\neq i}^{t}), a^t_{i} \rangle, s^{t+1}, o^{t+1}) \cdot T(s^{t}, \langle \psi_{\neq i}(n_{\neq i}^{t}), a^t_{i} \rangle, s^{t+1}).
  \intertext{The formula above does not raise the same issue as $\tilde O^t_i$ is a {\em state} variable, and the observation variable $O^t_i$ now does not depend on the previous state given the new one and the action (\cf Fig.~\ref{fig:BRpomdp} (bottom right)). 
Also, we have the following observation function:}
  & O_e(a^t_i, e^{t+1}_i, o^{t+1}_i)
  \eqdef Pr( o^{t+1}_i | a^t_i, e^{t+1}_i) \\
  & = Pr( o^{t+1}_i | a^t_i, \langle s^{t+1}, n^{t+1}_{\neq i}, \tilde o^{t+1}_i \rangle) = \mathbf{1}_{o^{t+1}_i = \tilde  o^{t+1}_i},
  \intertext{and the trivial reward function:}
  & r_e(e^t, a^t_i) =  r(s^t, a^t_i,\psi_{\neq i}(n_{\neq i}^{t}) ).
\end{align*}

\paragraph{\uline{$e^t = \langle s^t, n^{t-1}_{\neq i}, \tilde o^t_{\neq i} \rangle$ ?} } 
In this work, we have also considered a third choice for the extended state,
defined as
$e^t = \langle s^t, n^{t-1}_{\neq i}, \tilde o^t_{\neq i} \rangle$,
where $\tilde O^t_{\neq i}$ is a state variable (not an observation
variable) corresponding to other agents' observations at time $t$:
\begin{align*}
  & T_e(e^t, a^t_i, e^{t+1}) 
  \eqdef Pr(e^{t+1}| e^t, a^t_i) \\
  & = Pr(s^{t+1}, n_{\neq i}^{t}, \tilde o_{\neq i}^{t+1}| s^{t}, n_{\neq i}^{t-1}, \tilde o_{\neq i}^{t}, a^t_{i}) \\
  & = Pr(\tilde o_{\neq i}^{t+1}| s^{t+1}, n_{\neq i}^{t},  s^{t}, n_{\neq i}^{t-1}, \tilde o_{\neq i}^{t}, a^t_{i}) \\
  & \quad \cdot Pr(s^{t+1}, n_{\neq i}^{t}|  s^{t}, n_{\neq i}^{t-1}, \tilde o_{\neq i}^{t}, a^t_{i}) \\
  & = Pr(\tilde o_{\neq i}^{t+1}| s^{t+1}, n_{\neq i}^{t},  s^{t}, n_{\neq i}^{t-1}, \tilde o_{\neq i}^{t}, a^t_{i}) \\
  & \quad \cdot Pr(s^{t+1} | n_{\neq i}^{t},  s^{t}, n_{\neq i}^{t-1}, \tilde o_{\neq i}^{t}, a^t_{i}) \\
  & \quad \cdot Pr(n_{\neq i}^{t} | s^{t}, n_{\neq i}^{t-1}, \tilde o_{\neq i}^{t}, a^t_{i} ) \\
  & = Pr(s^{t+1}| s^{t}, n_{\neq i}^{t}, a^t_{i}) 
  \cdot Pr( n_{\neq i}^{t}| n_{\neq i}^{t-1}, \tilde o_{\neq i}^{t}) \\
  & \quad \cdot Pr(\tilde o_{\neq i}^{t+1}|s^{t+1}, n_{\neq i}^{t}, a^t_{i}) \\
  & =  Pr(s^{t+1}| s^{t}, n_{\neq i}^{t}, a^t_{i}) 
  \cdot Pr( n_{\neq i}^{t}| n_{\neq i}^{t-1},\tilde o_{\neq i}^{t}) \\
  & \quad \cdot \sum_{\tilde o^{t+1}_i}Pr(\tilde o_{\neq i}^{t+1},\tilde o^{t+1}_i|s^{t+1}, n_{\neq i}^{t}, a^t_{i}) \\
  & = T(s^{t}, \langle \psi_{\neq i}(n_{\neq i}^{t}), a^t_{i} \rangle, s^{t+1}) 
  \cdot \prod_{j\neq i} \eta(n_j^{t-1}, o_j^{t}, n_j^{t}) \\
  & \quad \cdot \sum_{\tilde o^{t+1}_{i}}O(\langle \psi_{\neq i}(n_{\neq i}^{t}), a^t_{i} \rangle, s^{t+1}, \langle \tilde o_{\neq i}^{t+1}, \tilde o^{t+1}_i \rangle).
  \intertext{The formula above does not raise issues. 
As illustrated by Fig.~\ref{fig:BRpomdp} (bottom right), the actual observation variable $O_i^{t+1}$ is independent on
    $E^t$ given $E^{t+1}$ and $A^t$.
Also, we have the following observation function:}
  & O_e(a^t_i, e^{t+1}_i, o^{t+1}_i)
  \eqdef Pr( o^{t+1}_i | a^t_i, e^{t+1}_i) \\
  & = Pr( o^{t+1}_i | a^t_i, \langle s^{t+1}, n^{t}_{\neq i}, \tilde o^{t+1}_{\neq i} \rangle) \\
  & = \frac{Pr( \tilde o^{t+1}_{\neq i}, o^{t+1}_i,  s^{t+1}, n^{t}_{\neq i}, a^t_i)}{Pr(\tilde o^{t+1}_{\neq i}, s^{t+1}, n^{t}_{\neq i}, a^t_i)}\\
  & = \frac{Pr( \tilde o^{t+1}_{\neq i}, o^{t+1}_i | s^{t+1}, n^{t}_{\neq i}, a^t_i)}{\sum_{o^{t+1}_i}Pr(\tilde o^{t+1}_{\neq i}, o^{t+1}_i | s^{t+1} n^{t}_{\neq i}, a^t_i)}\\
  & = \frac{O(\langle \psi_{\neq i}(n_{\neq i}^{t}), a^t_{i} \rangle, s^{t+1},  \langle \tilde o_{\neq i}^{t+1},o^{t+1}_i \rangle)}{\sum_{ \tilde o^{t+1}_{i}}O(\langle \psi_{\neq i}(n_{\neq i}^{t}), a^t_{i} \rangle, s^{t+1}, \langle \tilde o_{\neq i}^{t+1},  \tilde o^{t+1}_i \rangle)}.
  \intertext{ We can see that the denominator is identical to the last part of the transition function. In practice, when computing transition probabilities, we will store the values for $\sum_{\tilde o^{t+1}_{i}}O(\langle \psi_{\neq i}(n_{\neq i}^{t}), a^t_{i} \rangle, s^{t+1}, \langle \tilde o_{\neq i}^{t+1}, \tilde o^{t+1}_i \rangle)$ so as to reuse them when computing the observation function. In the end, the reward function is obtained with:}
  & r_e(e^t, a^t_i) 
  = r(\langle s^{t}, n^{t-1}_{\neq i}, \tilde o^{t}_{\neq i} \rangle, a^t_i) \\
  & = \sum_{ n^{t}_{\neq i}} Pr(n^{t}_{\neq i}| n^{t-1}_{\neq i}, \tilde o^{t}_{\neq i}) \cdot r( s^{t}, n^{t-1}_{\neq i}, a^t_i) \\
  & = \sum_{ n^{t}_{\neq i}}  \Big( \prod_{j\neq i} \eta(n_j^{t-1}, \tilde o_j^{t}, n_j^{t}) \Big) \cdot r( s^{t}, \langle \psi_{\neq i}(n_{\neq i}^{t}), a^t_i \rangle).
\end{align*}

Different formalizations lead to different state spaces with different
sizes, so that the time and space complexities of the generation of
this best-response POMDP or of its solving will depend on the choice of formalization.
Which choice is best shall depend on the Dec-POMDP at hand, and
possibly on the current FSCs.
We have opted for the simplest of the two formalizations presented
above (which we call the ``MOMDP formalization'' because one state
variable is fully observed), but observed little differences in
practice in our experiments.

\begin{figure*}
  \centering

  \vspace{-1em}

  \begin{minipage}{.48\linewidth}
    \centering
    \includegraphics[width=.76\linewidth]{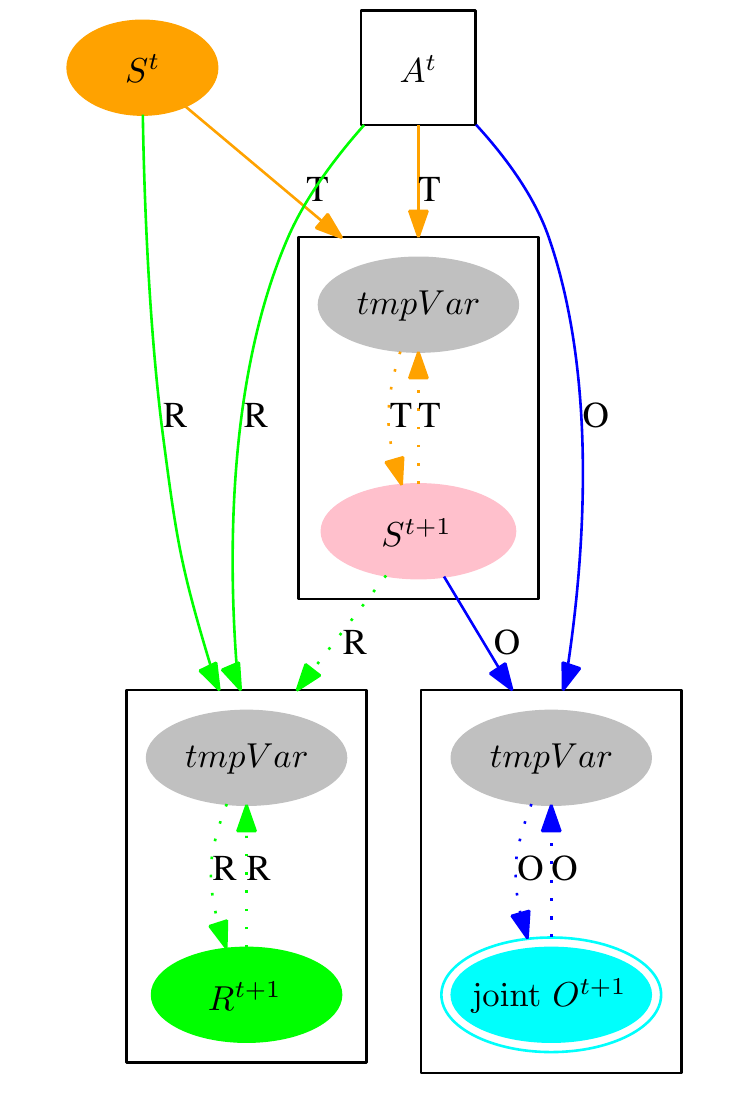}
  \end{minipage}
  \hfill \begin{minipage}{.48\linewidth}
    \centering
    \includegraphics[width=.86\linewidth]{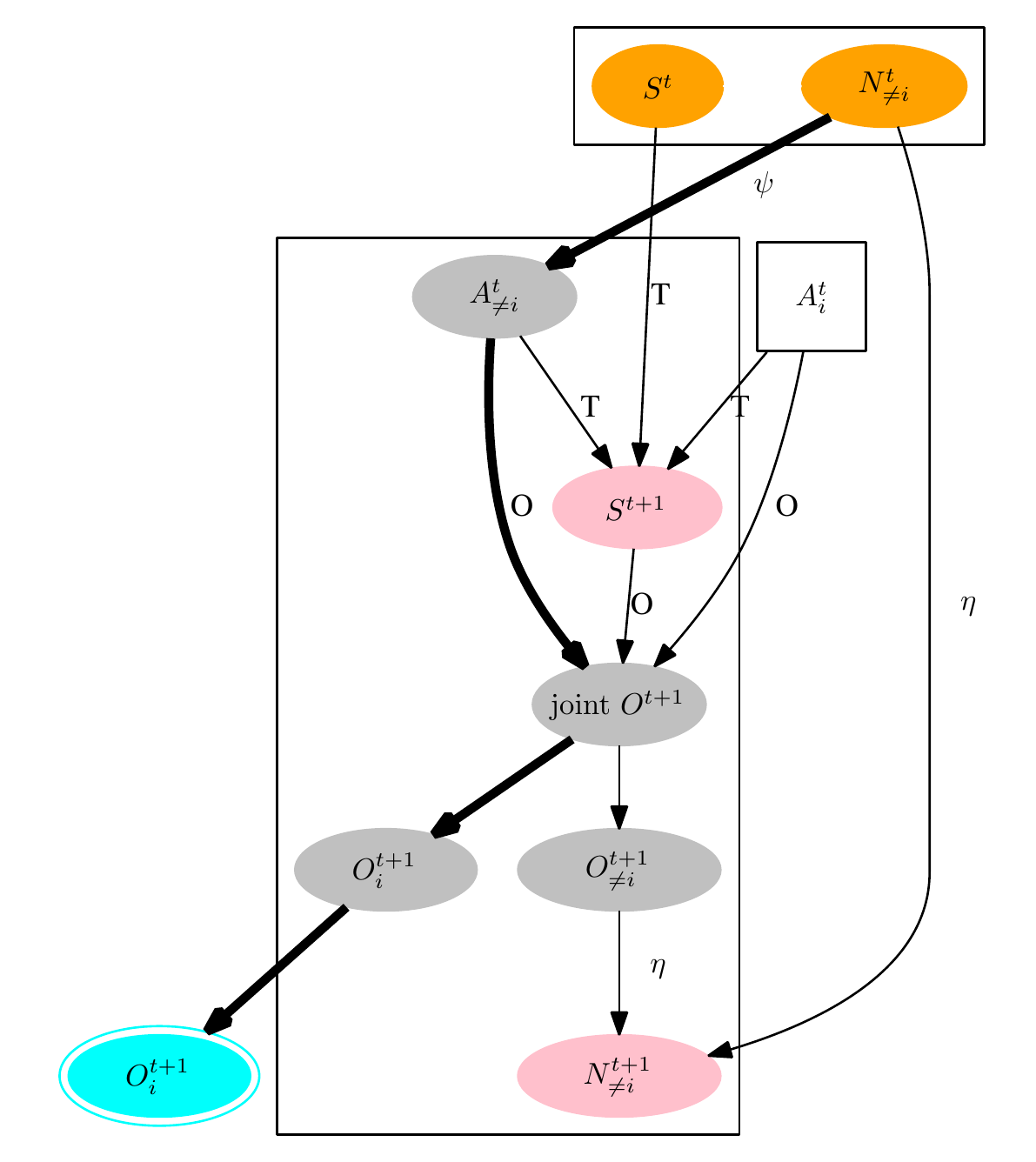}
  \end{minipage}

  \begin{minipage}{.48\linewidth}
    \centering
    \includegraphics[width=1.06\linewidth]{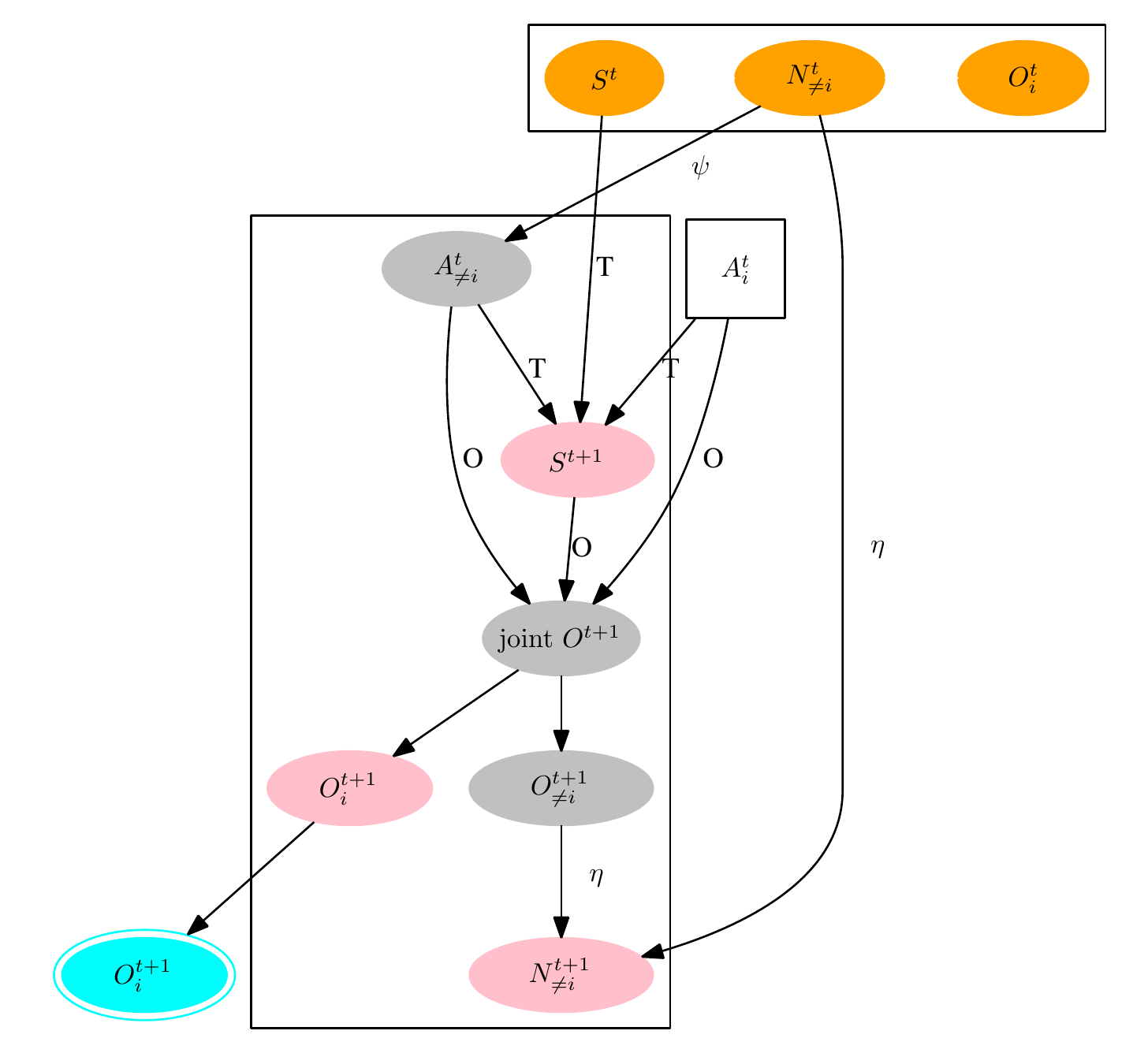}
  \end{minipage}
  \hfill
  \begin{minipage}{.48\linewidth}
    \centering
    \includegraphics[width=0.86\linewidth]{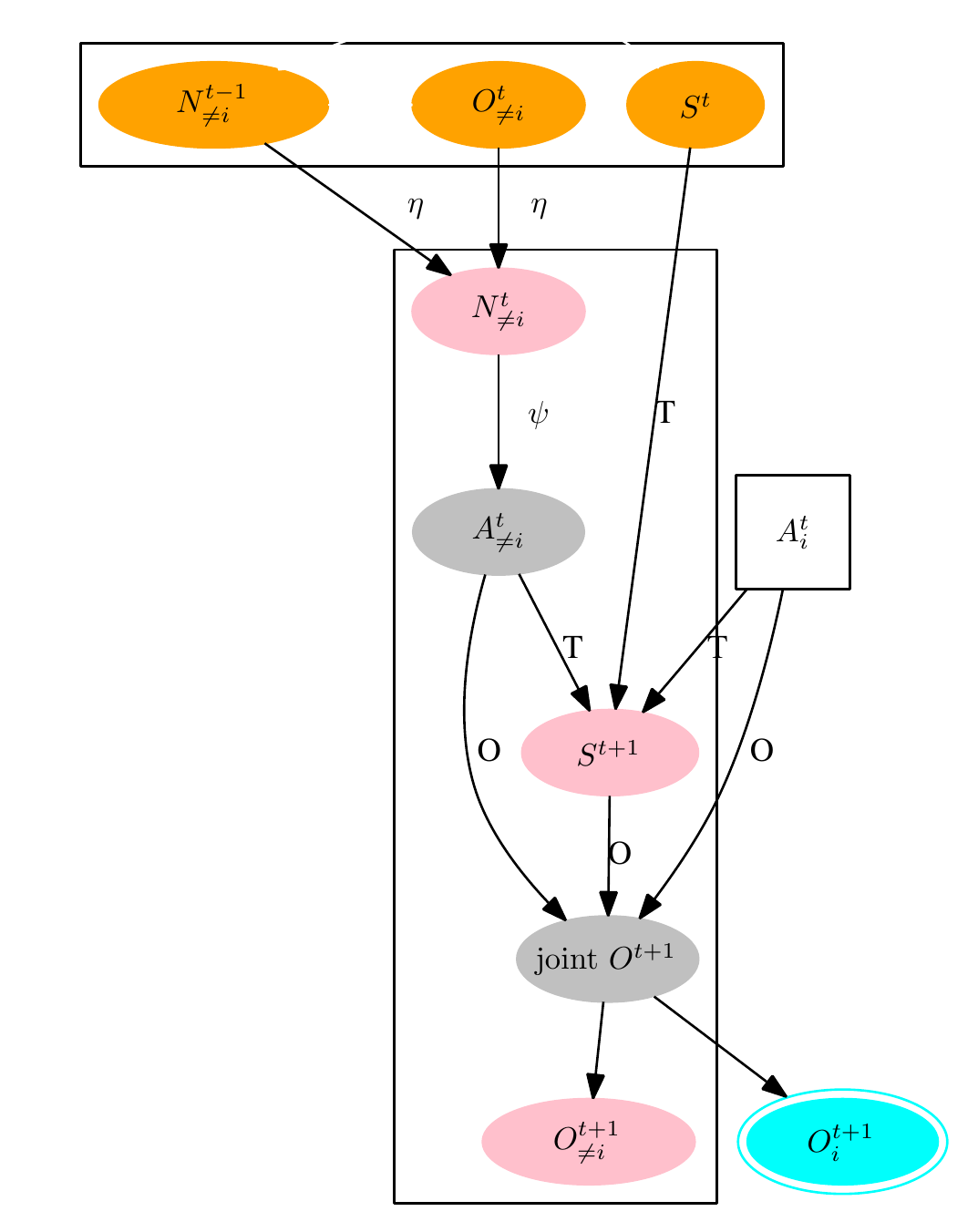}
  \end{minipage}

  \caption{
(top left) Standard POMDP dependencies, including intermediate variables used to compute the transition function (only); and candidate Best-Response POMDP formalizations with: (top right) $e^t_i \eqdef \langle s^t, n^t_{\neq i} \rangle$, (bottom left) $e^t_i \eqdef \langle s^t, n^t_{\neq i}, o^t_i \rangle$, and (bottom right) $e^t_i \eqdef \langle s^t, n^{t-1}_{\neq i}, o^t_{\neq i} \rangle$.
  }

  \label{fig:BRpomdp}
\end{figure*}

\section{Supplementary Figures}
\label{app:CompleteResults}

\def\scaleg{0.99}
\def\scalem{0.32}
\begin{figure*}\begin{minipage}{\scalem\linewidth}
       \centering
       \includegraphics[width=\scaleg\textwidth]{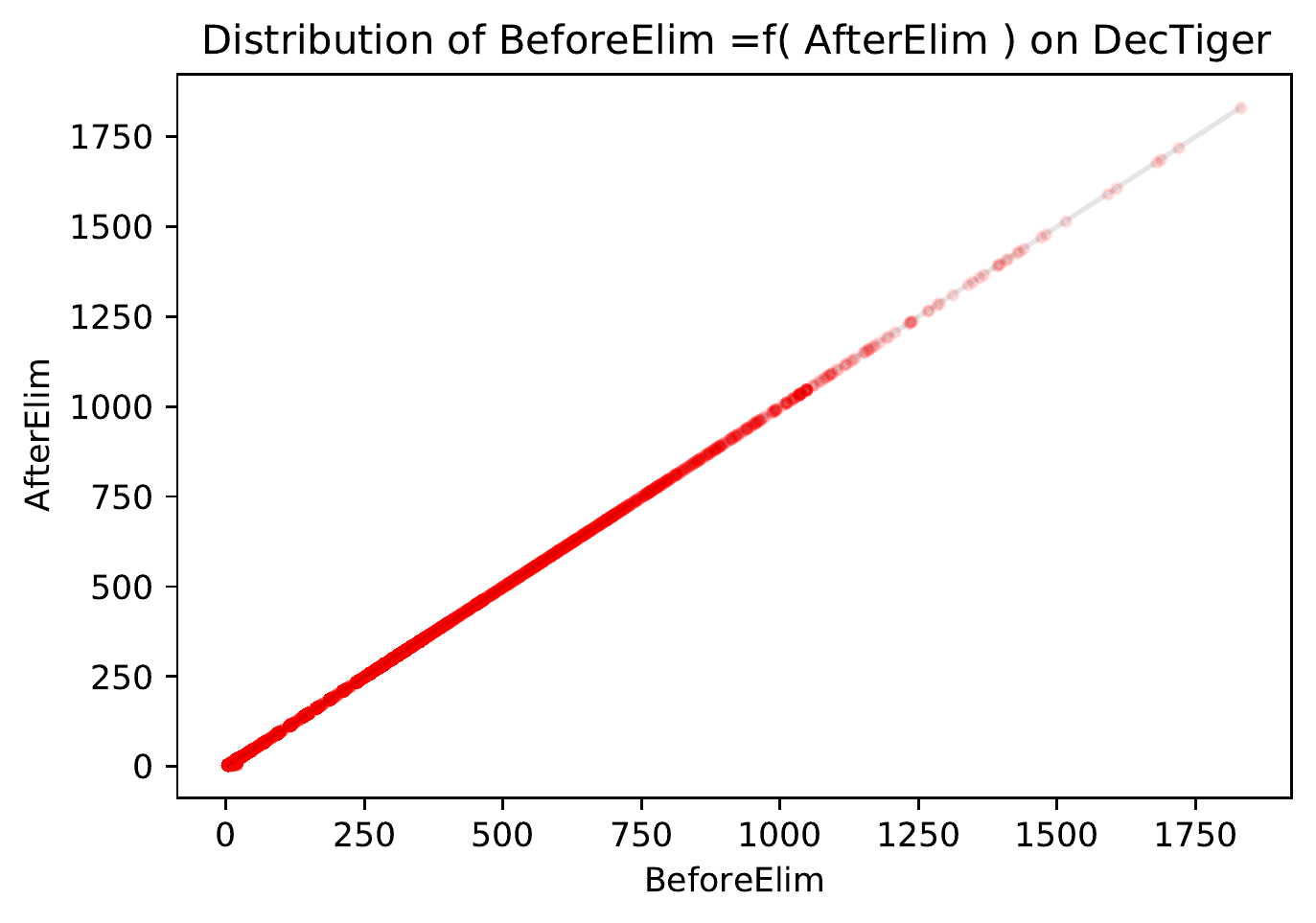}
     \end{minipage}
     \hfill
     \begin{minipage}{\scalem\linewidth}
       \centering
       \includegraphics[width=\scaleg\textwidth]{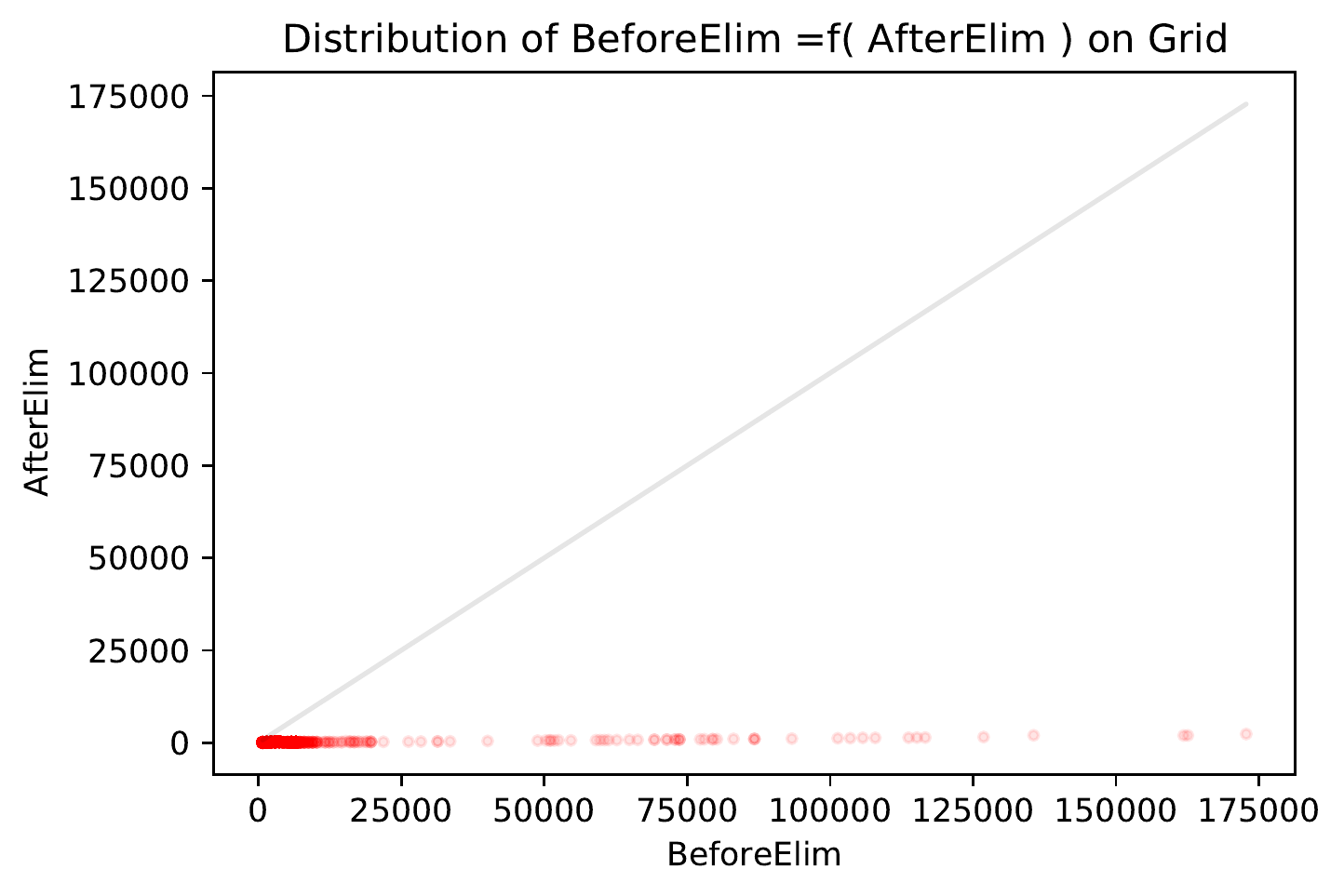}
     \end{minipage}
     \hfill
     \begin{minipage}{\scalem\linewidth}
       \centering
       \includegraphics[width=\scaleg\textwidth]{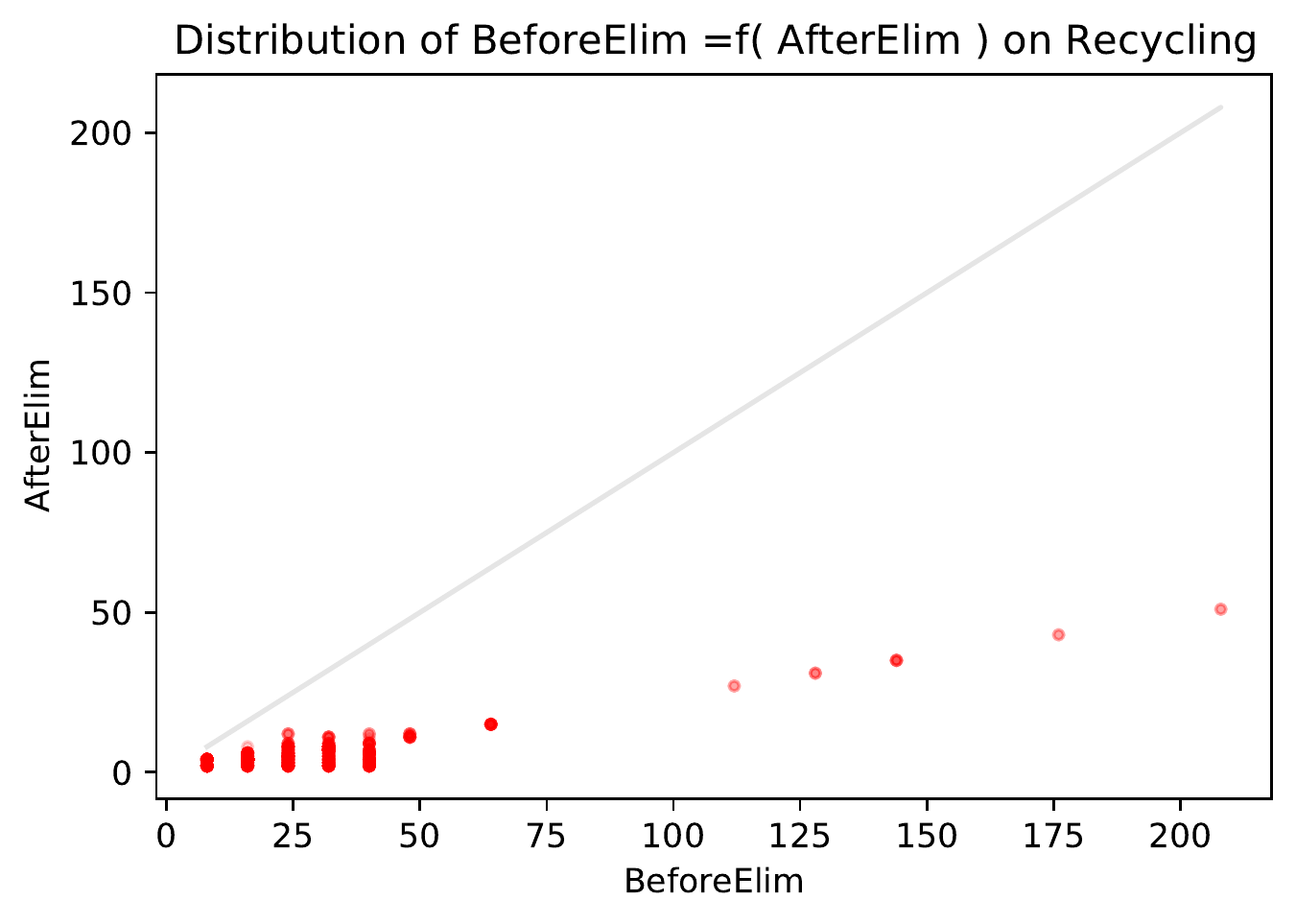}
     \end{minipage}
     \caption{Effect of the state elimination for the DecTiger, Grid and Recycling problems (from left to right).
Each figure plots, for each Best-Response POMDP obtained while executing \infJESP[(R-$1_1$)], the number of states of that POMDP after state elimination as a function of that number before state elimination.
}
     \label{Figure:StateElimination}
\end{figure*}

Regarding {\em state elimination} (as explained in Section~\ref{sec:BRPOMDP}), it turned out that, for the ``MOMDP'' best-response POMDP, the ratio of the initial number of (extended) states over its number of states after state elimination depends on the problem at hand but does not depend on the \infJESP run (see Figure~\ref{Figure:StateElimination}).
This ratio was 1 for the DecTiger problem, 50 for the Grid problem, and 5 for the Recycling problem.
These differences are due to the stochasticity of the observation process, which limits or even prevents state elimination. 
As it is currently done, state elimination is based on the probability of generating a specific observation from a state considering possible actions. 
However, for the DecTiger problem, every observation can be generated whatever the considered action, leading to no state elimination (contrary to the Grid and the Recycling problem).

\def\scaleg{0.99}
\def\scalem{0.32}
\begin{figure*}\begin{minipage}{\scalem\linewidth}
       \centering
       \includegraphics[width=\scaleg\textwidth]{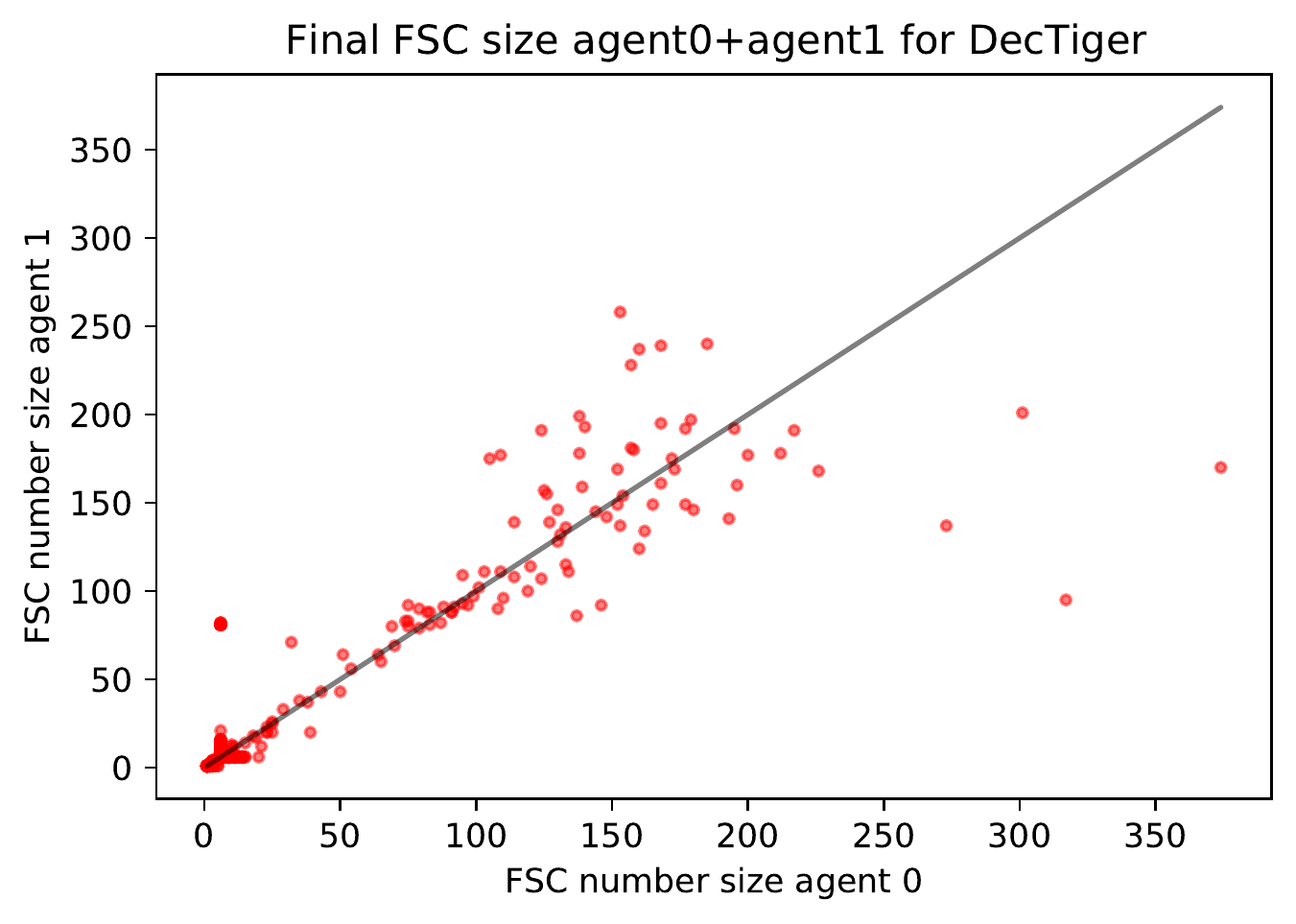}
     \end{minipage}
     \hfill
     \begin{minipage}{\scalem\linewidth}
       \centering
       \includegraphics[width=\scaleg\textwidth]{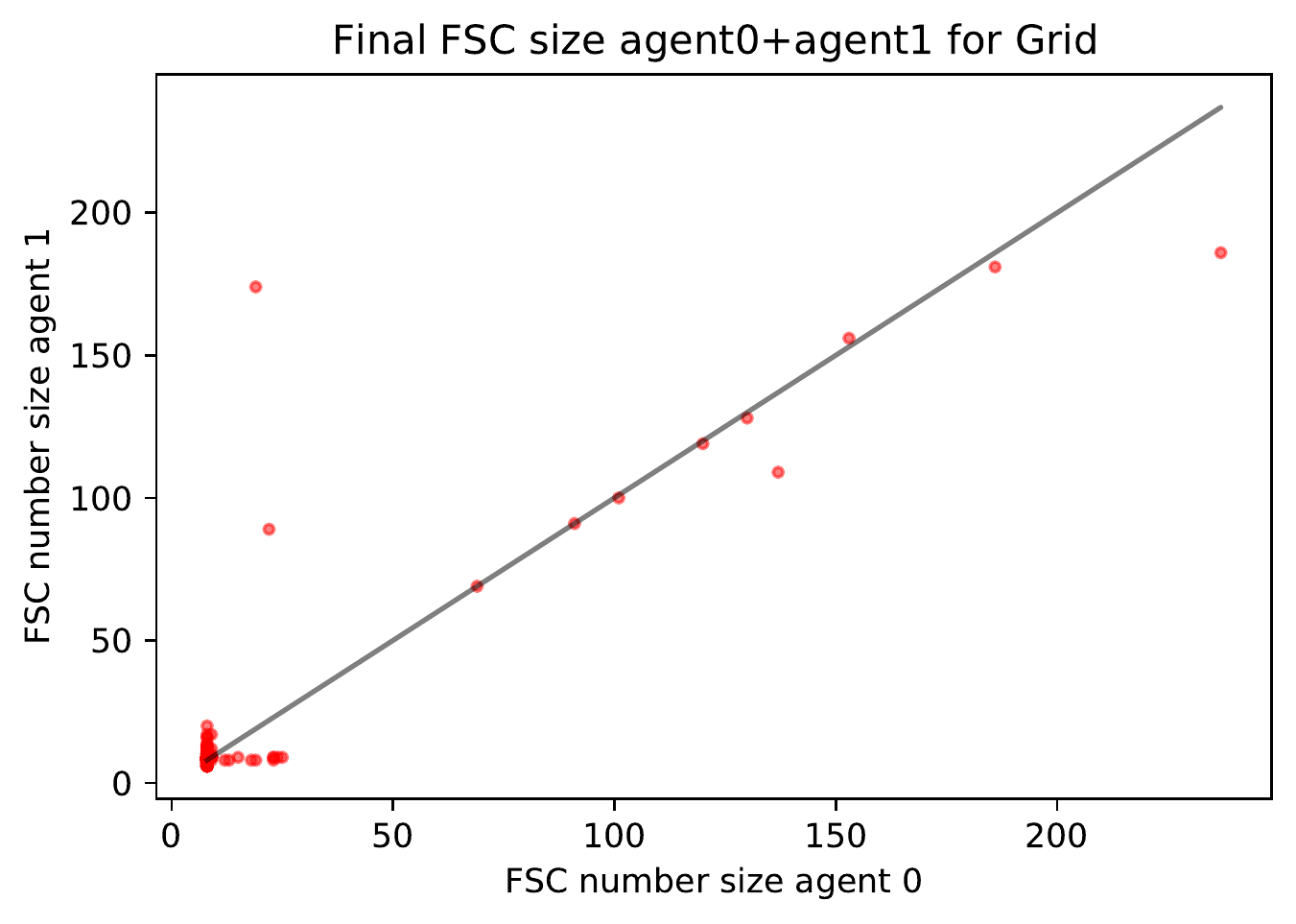}
     \end{minipage}
     \hfill
     \begin{minipage}{\scalem\linewidth}
       \centering
       \includegraphics[width=\scaleg\textwidth]{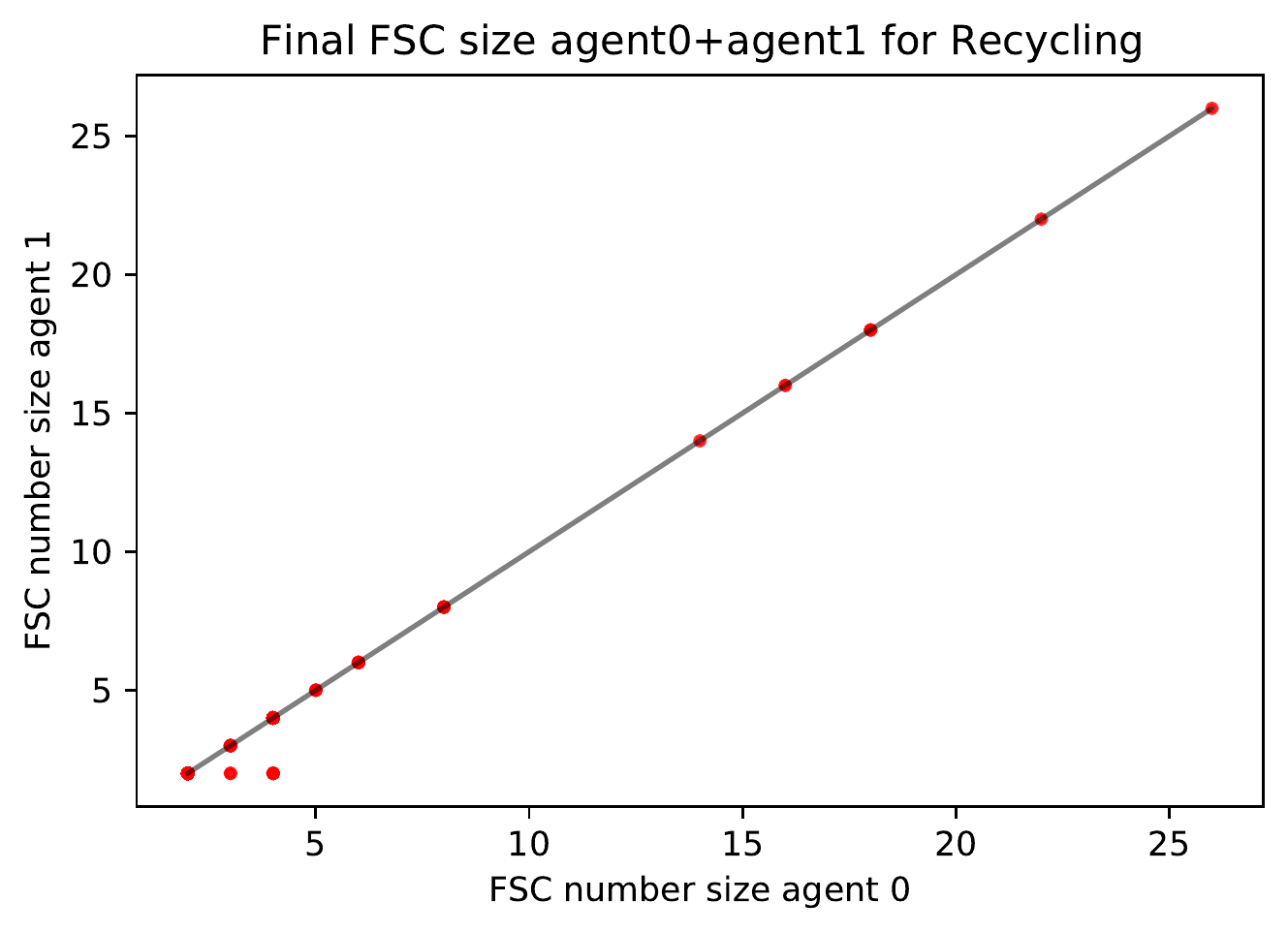}
     \end{minipage}
     \caption{Sizes of the final FSCs obtained with \infJESP[(R-$1_1$)] after convergence for the DecTiger, Grid and Recycling problems (from left to right).
Each dot corresponds to a pair containing the sizes of both agents' FSCs.
}
     \label{Figure:Final_FSC_Size}
\end{figure*}

Regarding {\em the size of the final FSCs} obtained after convergence of an \infJESP run with a random initialization (see Figure~\ref{Figure:Final_FSC_Size}), we also observed different behaviors.
Applied to the DecTiger problem, \infJESP leads to a very large distribution of the sizes of the final FSCs.
Applied to the Recycling problem, \infJESP  also leads to a large distribution of the sizes of the final FSC, but the sizes of the FSCs of both agents are symmetrical.
Applied to the Grid problem, \infJESP generates a large number of FSC pairs of size 10 with sometimes huge variation and with a tendency to be asymmetrical, the first optimized agent having a tendency to have more FSC nodes.
These results need to be further investigated in order to understand them clearly and see if it is possible to link them to the nature of the addressed domain.

\def\scaleg{0.99}
\def\scalem{0.32}
\begin{figure*}\begin{minipage}{\scalem\linewidth}
       \centering
       \includegraphics[width=\scaleg\textwidth]{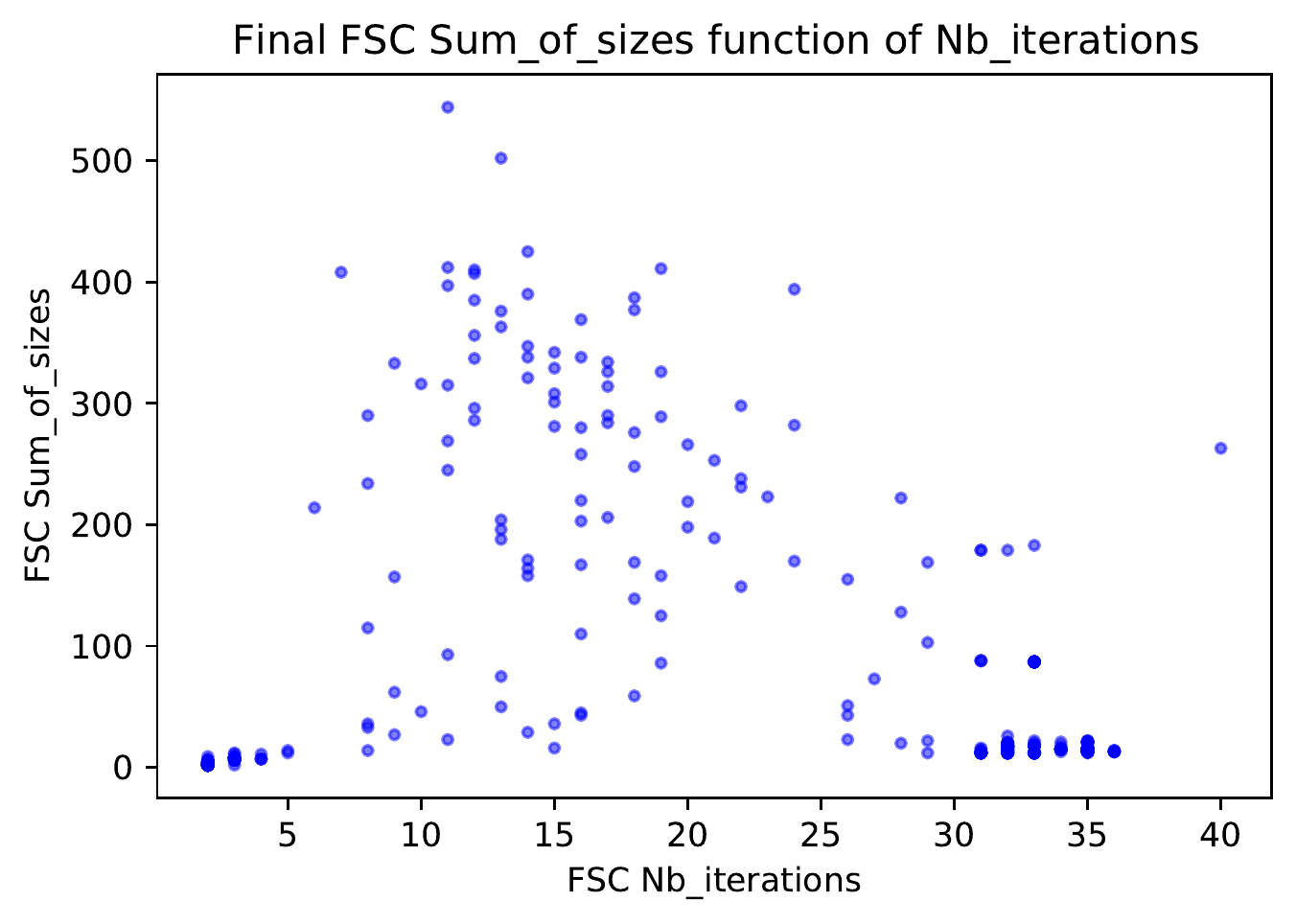}
     \end{minipage}
     \hfill
     \begin{minipage}{\scalem\linewidth}
       \centering
       \includegraphics[width=\scaleg\textwidth]{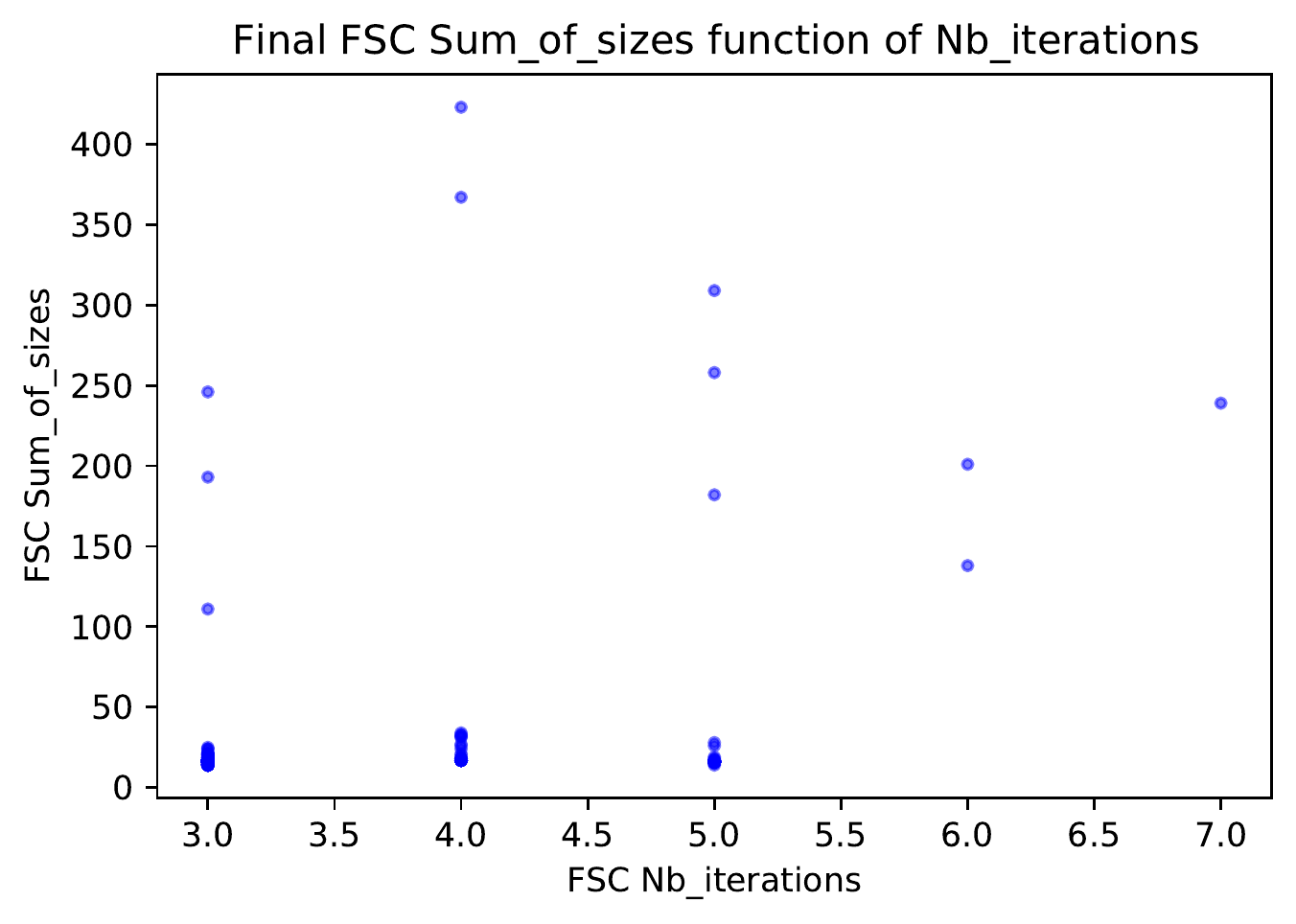}
     \end{minipage}
     \hfill
     \begin{minipage}{\scalem\linewidth}
       \centering
       \includegraphics[width=\scaleg\textwidth]{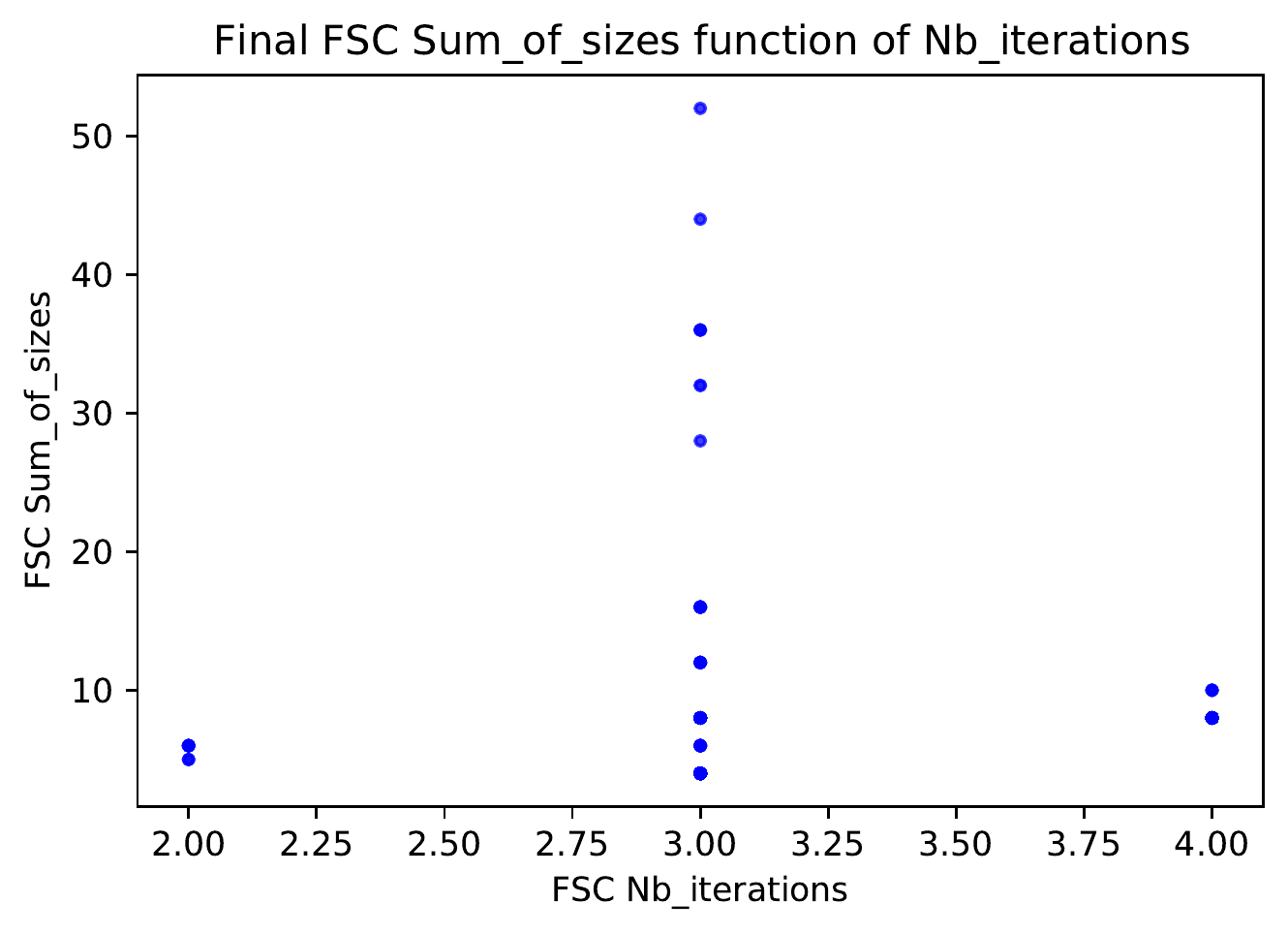}
     \end{minipage}
     \caption{Sum of the sizes of the final FSCs obtained with \infJESP[(R-$1_1$)] after convergence as a function of the required number of iterations for the DecTiger, Grid and Recycling problems (from left to right).
}
     \label{Figure:Final_FSCSize_Iteration}
\end{figure*}

\def\scaleg{0.99}
\def\scalem{0.32}
\begin{figure*}\begin{minipage}{\scalem\linewidth}
       \centering
       \includegraphics[width=\scaleg\textwidth]{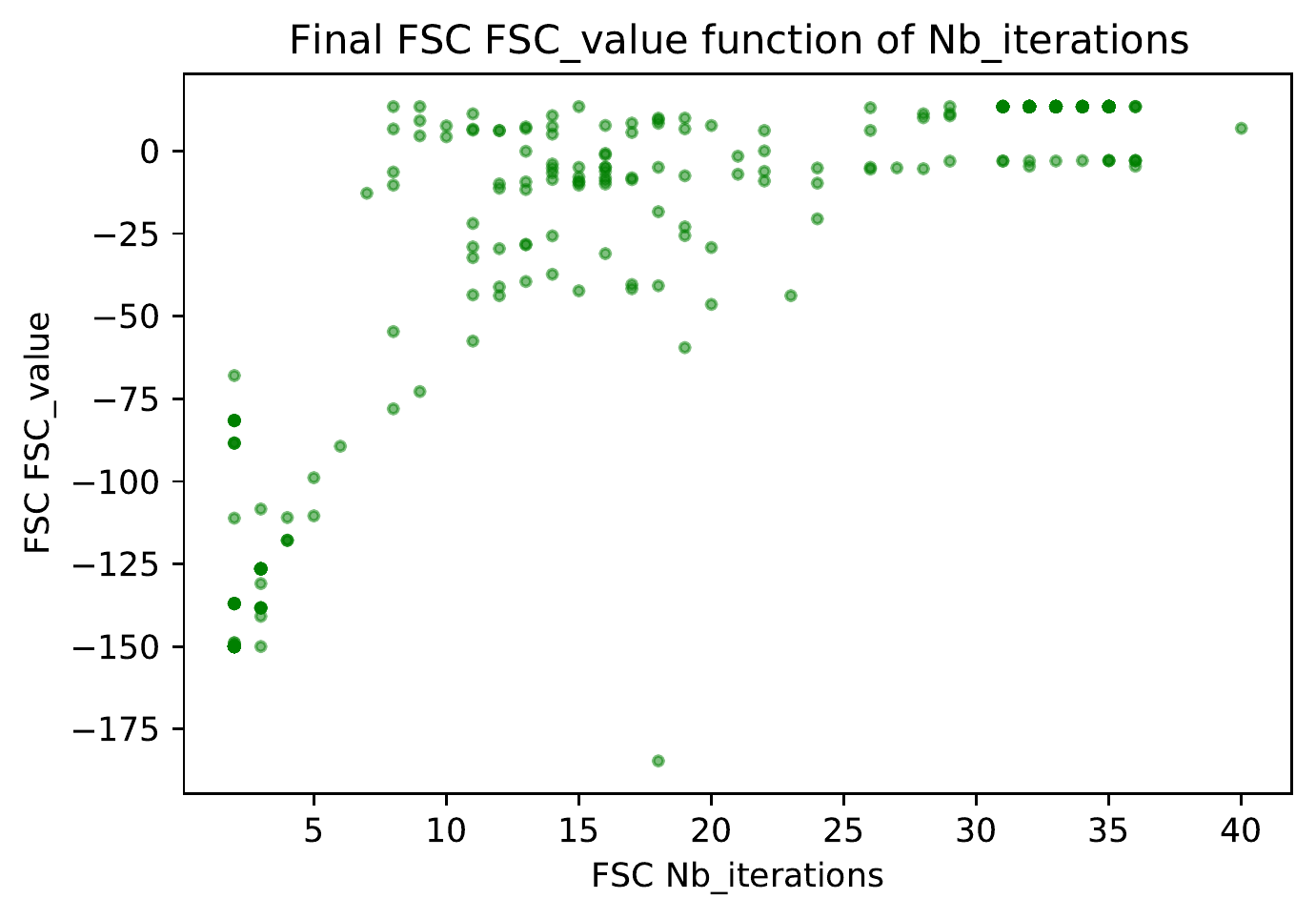}
     \end{minipage}
     \hfill
     \begin{minipage}{\scalem\linewidth}
       \centering
       \includegraphics[width=\scaleg\textwidth]{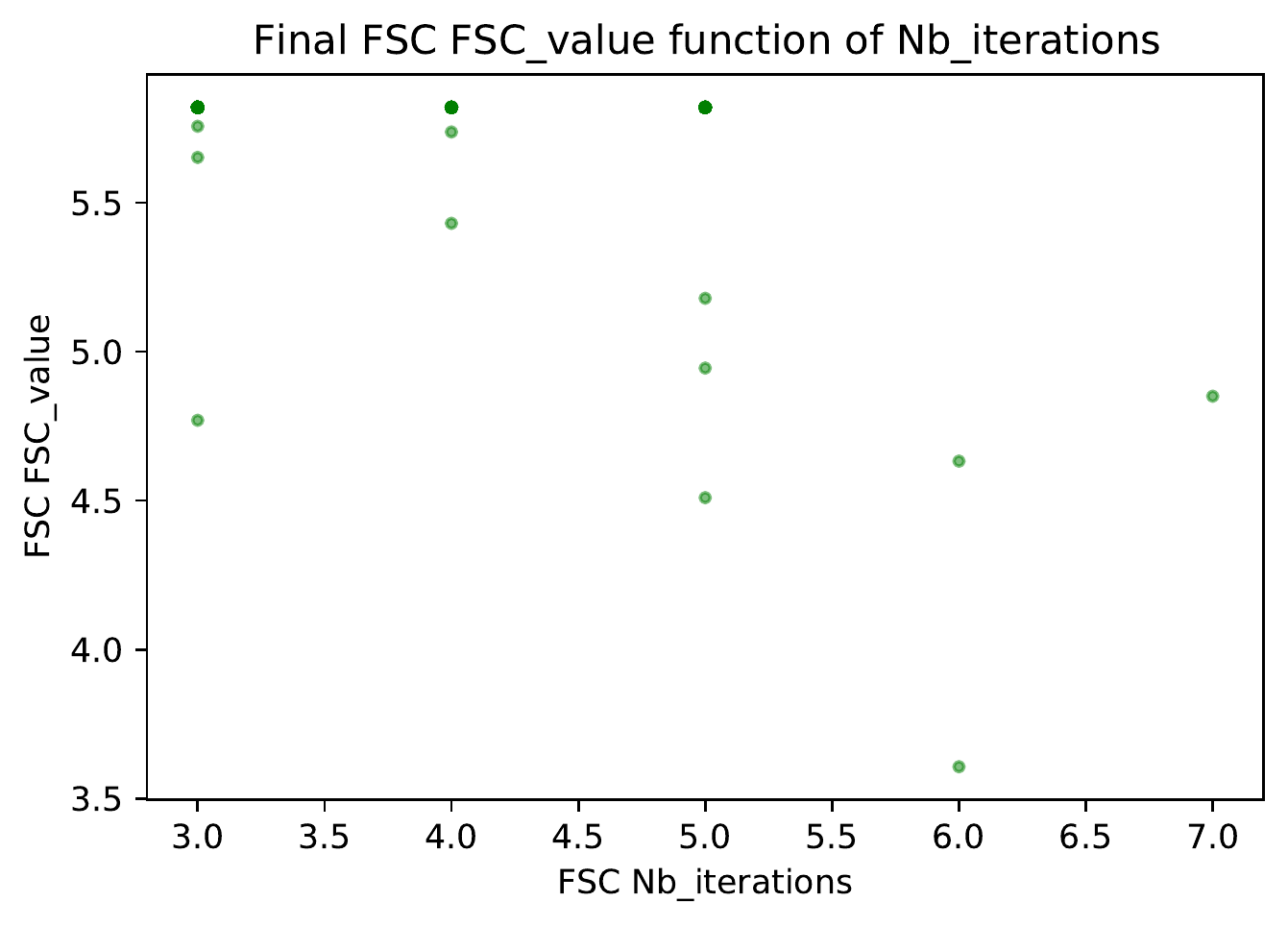}
     \end{minipage}
     \hfill
     \begin{minipage}{\scalem\linewidth}
       \centering
       \includegraphics[width=\scaleg\textwidth]{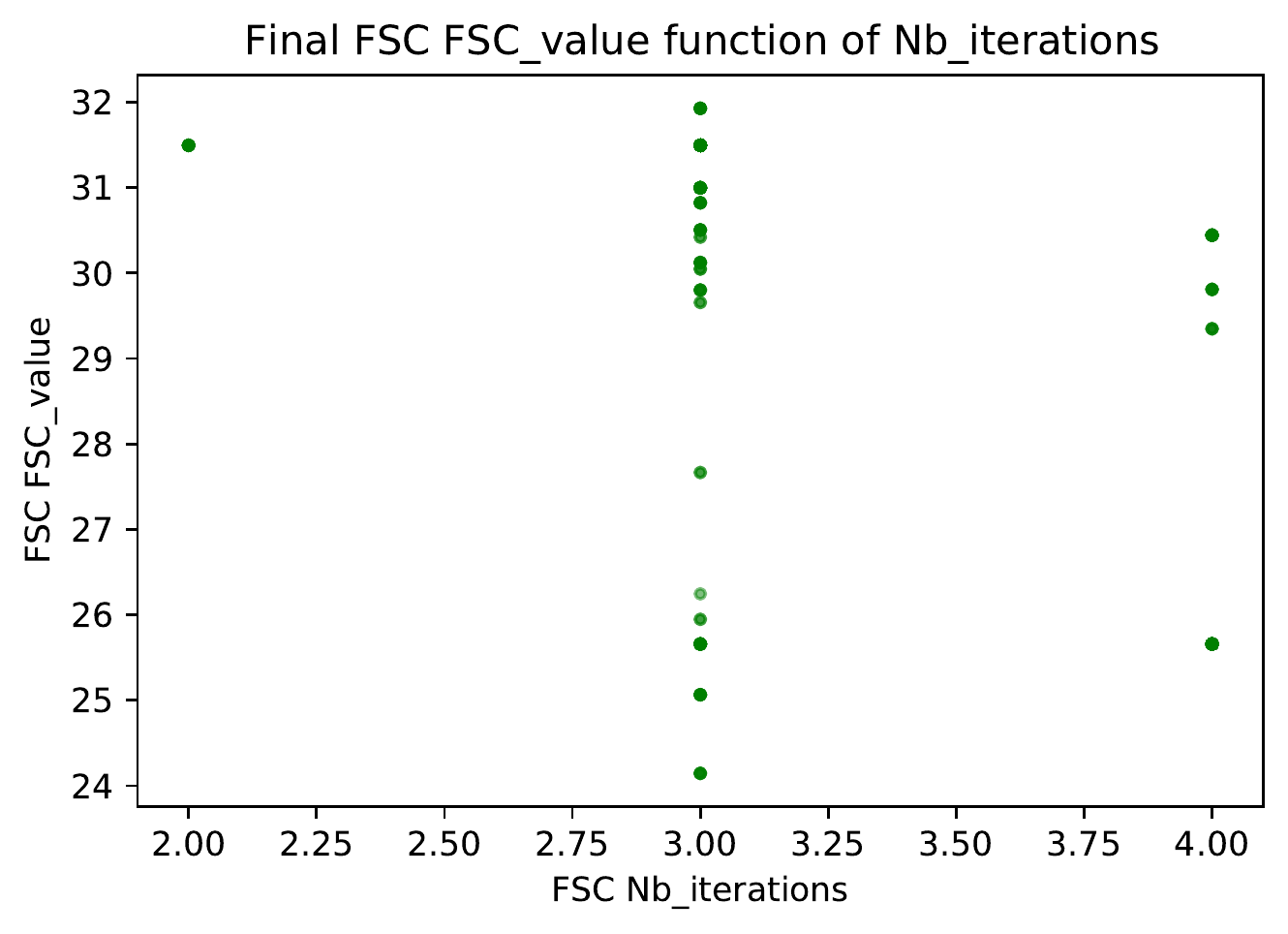}
     \end{minipage}
     \caption{Value of the final FSCs obtained with \infJESP[(R-$1_1$)] after convergence as a function of the required number of iterations for the DecTiger, Grid and Recycling problems (from left to right).
}
     \label{Figure:Final_Value_Iteration}
\end{figure*}

It must also be noted that {\em the sizes of the FSCs} do not monotonically increase during one \infJESP iteration (data not presented here). 
Sometimes the size of the FSC computed during one \infJESP improvement decreases, meaning that the solution to the best-response POMDP is a FSC with a smaller size but a higher value (as commonly observed in the DecTiger problem).
When looking at the FSCs obtained at the end of one \infJESP run, another observed phenomenon is that the more iterations required to reach the equilibrium, the smaller the sizes of the final FSCs (see Figure~\ref{Figure:Final_FSCSize_Iteration}).
But this does not mean that the associated value is higher (see Figure~\ref{Figure:Final_Value_Iteration}).

Finally, Figure~\ref{Figure:Final_Value_Size} presents the values of the obtained equilibria as a function of the sum of sizes of the FSCs obtained by \infJESP[(R-$1_1$)] after convergence.
It can be observed that small FSCs seem to be sufficient to generate a high value, opening new directions on combining \infJESP with FSC compression.

\newpage

\onecolumn

~~~ 

\def\scaleg{0.99}
\def\scalem{0.32}
\begin{figure}\begin{minipage}{\scalem\linewidth}
       \centering
       \includegraphics[width=\scaleg\textwidth]{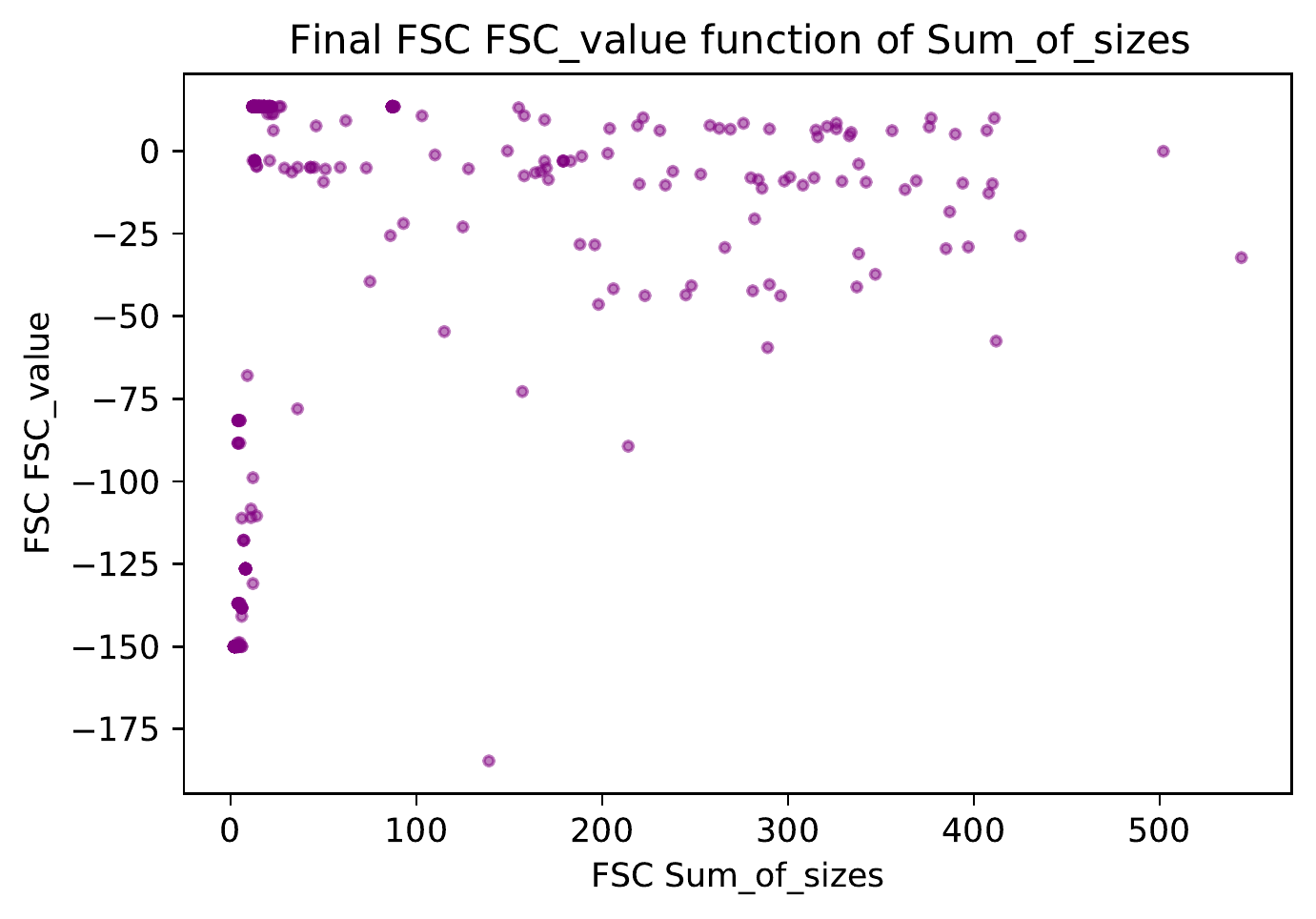}
     \end{minipage}
     \hfill
     \begin{minipage}{\scalem\linewidth}
       \centering
       \includegraphics[width=\scaleg\textwidth]{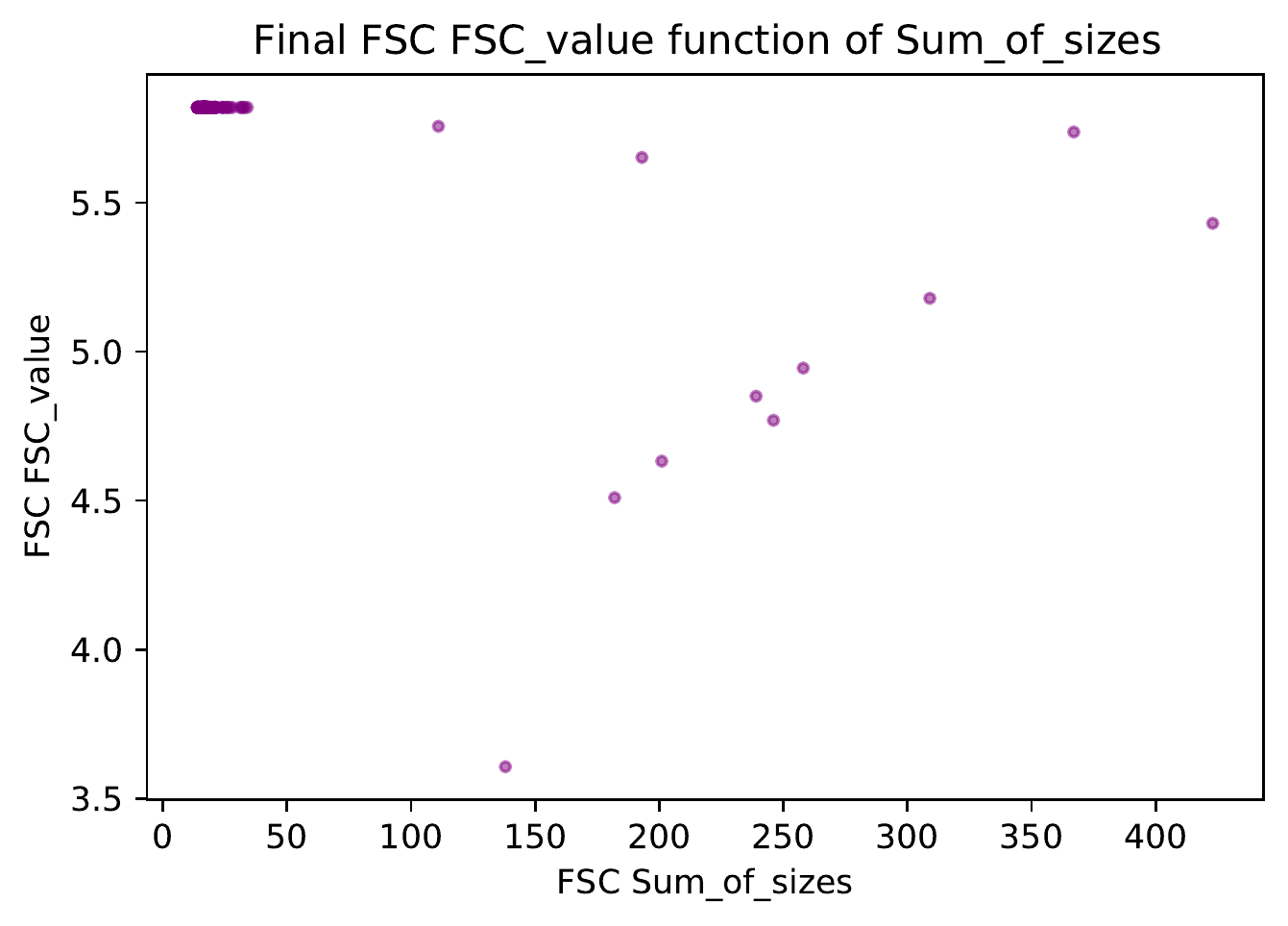}
     \end{minipage}
     \hfill
     \begin{minipage}{\scalem\linewidth}
       \centering
       \includegraphics[width=\scaleg\textwidth]{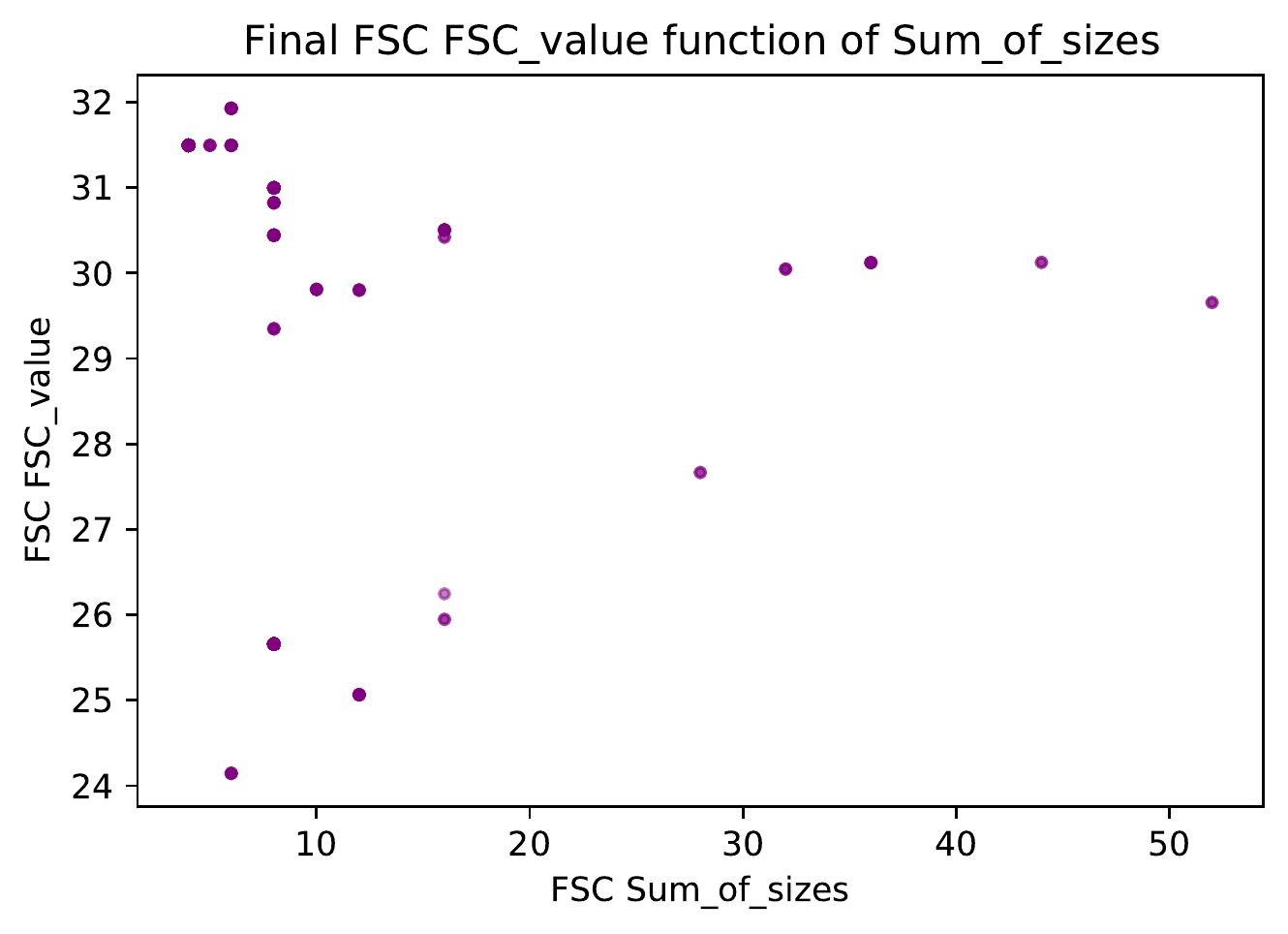}
     \end{minipage}
     \caption{Value of the final FSCs obtained with \infJESP[(R-$1_1$)] after convergence as a function of the sum of the sizes of final FSCs for the DecTiger, Grid and Recycling problems (from left to right).
     }
     \label{Figure:Final_Value_Size}
\end{figure}

 \fi

\end{document}